\def\isarxiv{1}

\def\paperTitle{Optimal Decay Spectra for Linear Recurrences}

\ifdefined\isarxiv
\documentclass[11pt]{article}
\usepackage[numbers]{natbib}
\else
\documentclass{article}
\usepackage[preprint]{neurips_2026}
\fi

\ifdefined\isarxiv

\usepackage{amsmath}
\usepackage{mathtools}
\usepackage{amsthm}
\usepackage{amssymb}
\usepackage{algorithm}
\usepackage{subfig}
\usepackage{algpseudocode}
\usepackage{graphicx}
\usepackage{grffile}
\usepackage{wrapfig,epsfig}
\usepackage{url}
\usepackage{xcolor}
\usepackage{epstopdf}

\usepackage{bbm}
\usepackage{dsfont}

\usepackage{booktabs}
\usepackage{multicol}
\usepackage{multirow}
\else 

\usepackage[utf8]{inputenc} 
\usepackage[T1]{fontenc}    
\usepackage{hyperref}       
\usepackage{url}            
\usepackage{booktabs}       
\usepackage{amsfonts}       
\usepackage{nicefrac}       
\usepackage{microtype}      
\usepackage{xcolor}         

\fi

\ifdefined\isarxiv

\else
\usepackage{amsmath}
\usepackage{mathtools}
\usepackage{amsthm}
\usepackage{amssymb}
\usepackage{algorithm}
\usepackage{subfig}
\usepackage{algpseudocode}
\usepackage{graphicx}
\usepackage{grffile}
\usepackage{wrapfig,epsfig}
\usepackage{epstopdf}
\usepackage{bbm}
\usepackage{dsfont}

\usepackage{multicol}
\usepackage{multirow}
\fi
 
\allowdisplaybreaks

\ifdefined\isarxiv

\usepackage{tikz}
\usepackage{hyperref}  
\hypersetup{colorlinks=true,citecolor=blue,linkcolor=blue} 
\usetikzlibrary{arrows}
\usepackage[margin=1in]{geometry}

\else
\definecolor{mydarkblue}{rgb}{0,0.08,0.45}
\hypersetup{colorlinks=true, citecolor=mydarkblue,linkcolor=mydarkblue}
\fi
 
\graphicspath{{./figs/}}

\theoremstyle{plain}
\newtheorem{theorem}{Theorem}[section]
\newtheorem{lemma}[theorem]{Lemma}
\newtheorem{definition}[theorem]{Definition}

\newtheorem{proposition}[theorem]{Proposition}
\newtheorem{corollary}[theorem]{Corollary}

\newtheorem{condition}[theorem]{Condition}

\newtheorem{remark}[theorem]{Remark}

\newcommand{\wh}{\widehat}

\newcommand{\R}{\mathbb{R}}

\renewcommand{\d}{\mathrm{d}}

\renewcommand{\hat}{\wh}

\DeclareMathOperator*{\E}{{\mathbb{E}}}

\DeclareMathOperator*{\Z}{\mathbb{Z}}

\DeclareMathOperator{\diag}{diag}

\newcommand{\calL}{\mathcal{L}}

\newcommand{\norm}[1]{\left\lVert#1\right\rVert}
\newcommand{\abs}[1]{\left\lvert#1\right\rvert}
\newcommand{\inner}[2]{\langle #1, #2 \rangle}

\DeclareMathOperator{\SSD}{SSD}
\DeclareMathOperator{\softplus}{Softplus}

\DeclareMathOperator{\sech}{sech}
\DeclareMathOperator{\clamp}{clamp}

\begin{document}

\ifdefined\isarxiv

\date{}
\title{\paperTitle}
\author{%
Yang Cao\\
\texttt{yang@silifen.com}
}

\else

\title{\paperTitle}

\author{%
  Yang Cao \\
  \texttt{yang@silifen.com} \\
}

\maketitle

\begin{center}
\vspace{-10pt}
\centering
\includegraphics[width=0.3\linewidth]{logo/silifen-dark.pdf}
\end{center}

\fi

\ifdefined\isarxiv
\begin{titlepage}
\maketitle


\begin{abstract}
Linear recurrent models offer linear-time sequence processing but often suffer from suboptimal long-range memory. We trace this to the decay spectrum: for $N$ channels, random initialization collapses the minimum spectral gap to $O(N^{-2})$, yielding sub-exponential error $\exp(-\Omega(N/\log N))$; linear spacing avoids collapse but degrades to $\exp(-O(N/\sqrt{T}))$, practically algebraic over long contexts. We introduce \textbf{Position-Adaptive Spectral Tapering (PoST)}, an architecture-agnostic framework combining two mechanisms: (1)~\emph{Spectral Reparameterization}, which structurally enforces geometrically spaced log-decay rates, proven minimax optimal at rate $O(\exp(-cN/\log T))$; and (2)~\emph{Position-Adaptive Scaling}, the provably unique mechanism that eliminates the \emph{scale mismatch} of static spectra (where only $N\log t/\log T$ of $N$ channels are effective at position $t$) by stretching the spectrum to the actual dependency range, sharpening the rate to $O(\exp(-cN/\log t))$. This scaling natively induces \emph{fractional invariance}: the impulse response becomes scale-free, with channels interpolating between relative and absolute temporal coordinates. PoST integrates into any diagonal linear recurrence without overhead. We instantiate it across Mamba-2, RWKV-7, Gated DeltaNet, Gated Linear Attention, and RetNet. Pre-training at 180M--440M scales shows consistent zero-shot language modeling improvements, significant long-context retrieval gains for Mamba-2 (MQAR and NIAH), and competitive or improved performance across other architectures. Code is available at \url{https://github.com/SiLifen/PoST}.
\end{abstract}
\thispagestyle{empty}
\end{titlepage}

{\hypersetup{linkcolor=black}
\tableofcontents
}
\newpage

\else

\begin{abstract}
Linear recurrent models offer linear-time sequence processing but often suffer from suboptimal long-range memory. We trace this to the decay spectrum: for $N$ channels, random initialization collapses the minimum spectral gap to $O(N^{-2})$, yielding sub-exponential error $\exp(-\Omega(N/\log N))$; linear spacing avoids collapse but degrades to $\exp(-O(N/\sqrt{T}))$, practically algebraic over long contexts. We introduce \textbf{Position-Adaptive Spectral Tapering (PoST)}, an architecture-agnostic framework combining two mechanisms: (1)~\emph{Spectral Reparameterization}, which structurally enforces geometrically spaced log-decay rates, proven minimax optimal at rate $O(\exp(-cN/\log T))$; and (2)~\emph{Position-Adaptive Scaling}, the provably unique mechanism that eliminates the \emph{scale mismatch} of static spectra (where only $N\log t/\log T$ of $N$ channels are effective at position $t$) by stretching the spectrum to the actual dependency range, sharpening the rate to $O(\exp(-cN/\log t))$. This scaling natively induces \emph{fractional invariance}: the impulse response becomes scale-free, with channels interpolating between relative and absolute temporal coordinates. PoST integrates into any diagonal linear recurrence without overhead. We instantiate it across Mamba-2, RWKV-7, Gated DeltaNet, Gated Linear Attention, and RetNet. Pre-training at 180M--440M scales shows consistent zero-shot language modeling improvements, significant long-context retrieval gains for Mamba-2 (MQAR and NIAH), and competitive or improved performance across other architectures. Code is available at \url{https://github.com/SiLifen/PoST}.
\end{abstract}

\fi



\section{Introduction}
\label{sec:intro}

Sequence models are the foundation of modern language processing.  Given a growing sequence of tokens $(x_1, x_2, \ldots)$, the model must predict the next token $x_{t+1}$ using information retained from the entire history $x_1, \ldots, x_t$.  The core challenge is \emph{long-range memory}: as the sequence grows, the model must retain information from increasingly distant positions while processing each new token in bounded time.

Transformer-based architectures~\citep{vsp+17} solve this by explicitly attending to all prior tokens via a key--value cache, but at a quadratic cost in sequence length.  \emph{Linear recurrent models}, including State Space Models (SSMs)~\citep{ggr22, gd23, dg24, llc+26}, RWKV~\citep{pbm+23, pga+24, pzg+25}, gated linear recurrences~\citep{dmr+24}, and linear-attention variants~\citep{syc+23, yws+24}, offer an alternative: the entire context is compressed into a fixed-size latent state $h_t \in \R^N$, and each step updates this state in $O(N)$ time.  The memory horizon is determined by the \emph{decay spectrum}, the collection of per-channel decay rates in the diagonal state update, which controls how quickly each ``memory channel'' forgets past inputs.

Yet linear recurrent models trained at context length $T$ tend to degrade sharply at longer contexts.  We trace this fragility to two independent failure modes, one at initialization and one over long contexts, and address both.

\paragraph{Contributions.}
We propose \textbf{Position-Adaptive Spectral Tapering (PoST)}, an architecture-agnostic framework for scale-free sequential memory. Our core contributions are:
\begin{itemize}
    \item \textbf{Information-Theoretic Blueprint for Sequence Memory.} We establish a design blueprint based on the logarithmic equipartition of information in natural data. We show that ideal memory channels should be distributed geometrically, with timescales systematically spanning from a single token up to the full observed context length.
    \item \textbf{Structural Guarantee via Spectral Reparameterization.} We diagnose the failure modes of existing models: random initializations suffer from sub-exponential approximation errors $\exp(-\Omega(N/\log N))$ due to a severe contraction of the minimum spectral gap $O(N^{-2})$, while linearly spaced grids suffer from exponentially degraded approximation bounds $\exp(-O(N/\sqrt{T}))$ over long contexts. In response, we introduce \emph{Spectral Reparameterization}, a mechanism that structurally guarantees geometrically spaced decay rates. We prove this configuration achieves minimax-optimal exponential approximation for long-range power-law dependencies.
    \item \textbf{Dynamic Mechanism via Position-Adaptive Scaling.} We quantify the \emph{scale mismatch} of static spectra: at position $t$, only $N\log t/\log T$ of $N$ channels contribute, wasting a fraction $1 - \log t/\log T$ of the spectrum. We derive Position-Adaptive Scaling as the provably unique continuous mechanism that eliminates this waste, sharpening the approximation bound from $O(\exp(-cN/\log T))$ to $O(\exp(-cN/\log t))$ at every position. This unique scaling natively induces \emph{fractional invariance}: the model's impulse response becomes scale-free, with channels smoothly interpolating between relative and absolute temporal coordinates, all without computational overhead.
\end{itemize}

We evaluate PoST across Mamba-2~\citep{dg24}, RWKV-7~\citep{pzg+25}, Gated DeltaNet~\citep{ykh25}, Gated Linear Attention (GLA)~\citep{yws+24}, and RetNet~\citep{syc+23}. Pre-training at 180M and 440M scales demonstrates that PoST consistently improves zero-shot language modeling, yields significant gains in long-context retrieval (MQAR and Needle-In-A-Haystack) for Mamba-2 during length extrapolation, and delivers competitive or improved performance across other architectures. Our code is open-sourced at \url{https://github.com/SiLifen/PoST}.

Table~\ref{tab:theory_comparison} summarizes the theoretical landscape governing timescale approximation across the three initialization paradigms.

\begin{table}[h]
    \centering
    \caption{\textbf{Theoretical Landscape of Decay Spectra.} Approximation bounds for power-law kernels across parameterization strategies. Geometric spacing achieves the optimal spectral \emph{shape} but leaves the scale fixed to the design length~$T$; PoST eliminates this scale mismatch via position-adaptive scaling, sharpening the exponent from $\log T$ to $\log t$.}
    \label{tab:theory_comparison}
    \begin{tabular}{lcc}
        \toprule
        Paradigm & Min.\ Spectral Gap & Minimax Error (Power-Law Kernel) \\
        \midrule
        Random & $O(N^{-2})$ & $\exp\left(-\Omega\left(\frac{N}{\log N}\right)\right)$ \\[3pt]
        Linear & $\Theta(1)$ & $\Omega(N^{-\beta})$ \\[3pt]
        Geometric (static) & $O(1/N)$ & $O\left(\exp\left(-c \frac{N}{\log T}\right)\right)$ \\[3pt]
        \midrule
        PoST (Ours) & $O(1/N)$ & $O\left(\exp\left(-c \frac{N}{\log t}\right)\right)$ \\
        \bottomrule
    \end{tabular}
\end{table}

\paragraph{Roadmap.}
Section~\ref{sec:preliminaries} introduces the diagonal linear recurrence framework shared by all target architectures.
Section~\ref{sec:theory} introduces the scale-free information-theoretic model and establishes the geometric design blueprint.
Section~\ref{sec:methodology} diagnoses the failure modes of random and linear initializations, introduces the two-component PoST framework, establishes the minimax optimality of geometric spacing and the uniqueness of position-adaptive scaling, and derives the resulting scale-free impulse response.
Section~\ref{sec:post} gives architecture-specific instantiations for Mamba-2, RWKV-7, Gated DeltaNet, GLA, and RetNet.
Section~\ref{sec:experiments} reports experiments on MQAR, zero-shot language modeling, and Needle-In-A-Haystack retrieval.
Section~\ref{sec:conclusion} concludes.

\section{Preliminaries}
\label{sec:preliminaries}

This section introduces the mathematical framework underlying all subsequent results.


\subsection{Sequence Modeling and Autoregressive Prediction}
\label{subsec:sequence_modeling}

A \emph{sequence model} maps a history of observed tokens $(x_1, x_2, \ldots, x_{t-1})$ to a probability distribution over the next token~$x_t$.  In modern language models, the dominant paradigm is \emph{autoregressive prediction}: at each position~$t$, the model reads a single new token~$x_t$, updates an internal state, and outputs a prediction of $x_{t+1}$.

The fundamental challenge is \emph{memory}: to predict well, the model must retain relevant information from arbitrarily far back in the sequence.  Two broad families address this.  \emph{Transformers}~\citep{vsp+17} attend to all previous tokens via a key--value cache, offering unbounded memory at $O(T^2)$ computational cost for a sequence of length~$T$.  \emph{Linear recurrent models}, including State Space Models (SSMs)~\citep{ggr22, gd23, dg24}, gated linear recurrences~\citep{dmr+24}, and linear-attention variants~\citep{syc+23, yws+24}, compress the entire history into a fixed-size hidden state, yielding linear-time processing at $O(N)$ per step.  The fixed state imposes a finite memory horizon governed by the \emph{decay spectrum}: the collection of per-channel decay rates that control how quickly each ``memory mode'' forgets past inputs.

This paper studies how to design the decay spectrum so that the fixed-size state retains long-range information as effectively as possible, in a manner that applies to \emph{all} linear recurrent models.


\subsection{Diagonal Linear Recurrences}
\label{subsec:linear_recurrence}

We define the general computational primitive shared by all architectures considered in this paper.

\begin{definition}[Diagonal Linear Recurrence]
\label{def:linear_recurrence}
A \emph{diagonal linear recurrence} is a sequence model whose hidden state $S_t \in \R^{N \times d}$ evolves as
\begin{align}
\label{eq:general_recurrence}
    S_t &= \diag(w_t) \cdot S_{t-1} + F(x_t, S_{t-1}), \qquad y_t = G(S_t, x_t),
\end{align}
where $w_t \in (0,1)^N$ is a vector of per-channel \emph{decay gates} (possibly data-dependent), and $F, G$ are architecture-specific input and output maps that do not depend on $w_t$.
\end{definition}

The decay vector $w_t$ controls how quickly each channel forgets past inputs.  Each architecture computes $w_t$ from a distinct combination of learnable base parameters and input-dependent modulations, but they all instantiate the same structural role.

\paragraph{Log-decay parameterization.}

Throughout this paper, we parameterize the decay spectrum via log-decay rates.  For a time-invariant base decay $w_k \in (0,1)$, we define:
\begin{align*}
    p_k := -\log w_k = \log(1/w_k) > 0.
\end{align*}
This maps the unit interval $(0,1)$ to the positive half-line $p_k \in (0, \infty)$ and makes the geometric structure of the spectrum explicit: a geometric progression $w_k = r^{k}$ corresponds to uniform spacing $p_k = -k \log r$ in log-decay space.

\begin{definition}[Timescale]
\label{def:timescale}
The \emph{timescale} of channel~$k$ is $\tau_k := 1/p_k$.  It controls the channel's effective memory horizon: the impulse response $w_k^t$ decays to $1/e$ at lag $\tau_k$.  The collection $\{\tau_1, \ldots, \tau_N\}$, the \emph{decay spectrum}, determines which temporal dependencies the model can represent.
\end{definition}

\paragraph{Spectral coherence.}
We introduce a measure of functional redundancy between memory channels.

\begin{definition}[Spectral Coherence]
\label{def:coherence}
For a diagonal linear recurrence with log-decay parameterization $p_k > 0$, the \emph{spectral coherence} between channels $i$ and $j$ is:
\begin{align*}
    \mu_{ij} := \frac{|\langle h_i, h_j \rangle|}{\|h_i\|_2 \|h_j\|_2} = \sech\left( \frac{|\ln p_i - \ln p_j|}{2} \right),
\end{align*}
where $h_k(s) = e^{-p_k s}$ is the impulse response of channel~$k$, and the inner product is taken in $L^2(\R_{\ge 0})$.  This identity is exact: $\langle h_i, h_j \rangle = \int_0^\infty e^{-(p_i+p_j)s}\d s = (p_i + p_j)^{-1}$ and $\|h_k\|_2^2 = (2p_k)^{-1}$, so
$\mu_{ij} = {2\sqrt{p_i p_j}}/({p_i + p_j}) = \sech\bigl((\ln p_i - \ln p_j)/2\bigr)$,
where the last step follows from $\sech(x) = 2/(e^{x}+e^{-x})$ with $x = \tfrac{1}{2}\ln(p_i/p_j)$.
\end{definition}

When $\mu_{ij} \to 1$, channels $i$ and $j$ become indistinguishable: their impulse responses span nearly the same subspace, wasting one degree of freedom in the state.  Controlling spectral coherence is thus a prerequisite for efficient spectrum design.


\subsection{Architecture Instantiations}
\label{subsec:architectures}

The diagonal linear recurrence~\eqref{eq:general_recurrence} is the common computational primitive underlying a wide range of modern sequence models.  These architectures differ in how they compute the decay gates $w_t$ from learnable parameters, but share the same diagonal decay structure that our theory addresses.

\begin{definition}[PoST-Compatible]
\label{def:post_compatible}
A diagonal linear recurrence is \emph{PoST-compatible} if its decay gates $w_t$ can be decomposed as
\begin{align*}
    \log w_{t,k} = h\bigl(d_{\mathrm{base},k}, x_t\bigr),
\end{align*}
where $d_{\mathrm{base}} \in \R^N$ are learnable base decay parameters and $h$ is an architecture-specific function.  The PoST modification replaces the independent parameterization of $d_{\mathrm{base}}$ with the PoST map (Definition~\ref{def:post_map}) and scales the effective log-decay by a position-adaptive factor.
\end{definition}

This decomposition is satisfied by all major diagonal linear recurrences, including Mamba~\citep{gd23,dg24,llc+26}, RWKV-7~\citep{pzg+25}, RetNet~\citep{syc+23}, GLA~\citep{yws+24}, and Gated DeltaNet~\citep{ykh25}.

\paragraph{Connection to continuous-time memory (HiPPO).}
\label{subsubsec:hippo}

State Space Models (SSMs)~\citep{ggr22, gd23, dg24} arrive at the diagonal linear recurrence~\eqref{eq:general_recurrence} via the discretization of a continuous-time ordinary differential equation (ODE). The theoretical foundation for this approach is the HiPPO framework~\citep{ggr20}.

\begin{definition}[HiPPO Continuous-Time Memory]
\label{def:hippo}
Given a continuous input signal $u(t) \in \R$ and a time-varying measure $\omega^{(t)}$ supported on the past $(-\infty, t]$, the continuous-time memory state $h(t) \in \R^N$ maintains the optimal $L^2$ projection coefficients of the history $u_{\le t}$ onto the basis of orthogonal polynomials associated with $\omega^{(t)}$. The optimal coefficients formally evolve via the linear ODE:
\begin{align}
\label{eq:hippo_ode}
    \dot{h}(t) = A h(t) + B u(t),
\end{align}
where the transition matrix $A \in \R^{N \times N}$ and input matrix $B \in \R^{N \times 1}$ are mathematically determined by the chosen measure $\omega^{(t)}$.
\end{definition}

For the canonical scaled Legendre measure (HiPPO-LegS), $A$ acts as a structured dense operator. Diagonal State Space Models (e.g., S4D~\citep{ggr22b}) systematically simplify this ODE by showing that $A$ can be replaced by a diagonal matrix $\Lambda = \diag(\lambda_1, \ldots, \lambda_N)$ without sacrificing the principled memory compression. For instance, the standard S4D-Real initialization sets:
\begin{align}
    \lambda_k = -(k + 1), \quad k \in \{0, \dots, N-1\}.
\end{align}

Discretizing this diagonal ODE~\eqref{eq:hippo_ode} with a sampling step $\Delta > 0$ transforms the continuous system into the discrete diagonal linear recurrence~\eqref{eq:general_recurrence}, producing analytic decay gates $w_k = e^{\lambda_k \Delta}$. In modern extensions like Mamba~\citep{gd23, dg24}, the sampling step $\Delta$ becomes input-dependent ($\Delta_{k,t} = \softplus(\mathtt{dt\_bias}_k + \mathtt{dt\_proj}(x_t)_k)$), yielding data-dependent decays $w_{k,t} = e^{\lambda_k \cdot \Delta_{k,t}}$. Our PoST framework and memory capacity theorems operate directly on the effective discrete decay spectrum $w_t$, meaning they naturally encompass these continuous-time origins as a special case.


\subsection{Notation}
\label{subsec:notation}
We denote the real numbers by $\R$, the integers by $\Z$, and the integer range $\{1, \ldots, N\}$ by $[N]$.  Vectors are lowercase ($x, h$) and matrices uppercase ($A, B$).  We write $\diag(\cdot)$ for a diagonal matrix formed from a vector, $\norm{\cdot}_2$ for the $\ell_2$ norm, $\abs{\cdot}$ for the absolute value, $\inner{\cdot}{\cdot}$ for the standard inner product, and $\circ$ for the Hadamard (element-wise) product.  $\E[\cdot]$ and $\Pr(\cdot)$ denote expectation and probability.  All logarithms are natural unless otherwise noted.  Model-specific quantities ($p_k$, $\tau_k$, $w_{t,k}$, $\alpha_k$, $\mu_{ij}$) are defined at the point of first use.

\begin{definition}[Softplus]
\label{def:softplus}
The softplus function $\softplus: \R \to \R_{>0}$ is defined as $\softplus(z) = \log(1 + e^z)$.  It provides a smooth, strictly positive approximation to the ReLU.
\end{definition}

\begin{definition}[Hyperbolic Secant]
\label{def:sech}
The hyperbolic secant $\sech: \R \to (0, 1]$ is defined as $\sech(x) = 2/(e^x + e^{-x})$.  It is an even function with $\sech(0) = 1$ and $\sech(x) \to 0$ as $|x| \to \infty$.
\end{definition}
\section{Theoretical Foundations of Scale-Free Memory}
\label{sec:theory}

Before designing a specific neural architecture, we first derive the optimal memory structure from first principles, independent of any model parameterization. What is the ideal decay spectrum for a linear recurrent model? Three modeling conditions on the statistical structure of sequential data uniquely determine both the \emph{shape} (geometric spacing) and the \emph{scale} (position-dependent growth) of the optimal timescale allocation. The resulting theoretical blueprint establishes the mathematical target that the methodology in Section~\ref{sec:methodology} aims to implement.


\subsection{Modeling Conditions}
\label{subsec:conditions}

We model the input as a wide-sense stationary stochastic process $\{x_t\}_{t \ge 1}$ with $\E[x_t] = 0$ and finite variance.  Its autocovariance function is $R: \R_+ \to \R$, $R(s) := \E[x_t  x_{t+s}]$.  We formalize three empirically grounded properties of natural sequential data that together determine the optimal spectral allocation.

\paragraph{Scale invariance of correlations.}
A large body of empirical work establishes that the correlation structure of natural language is approximately scale-free: the power spectral density follows a $1/f^\beta$ law across several decades of frequency~\citep{vc78, ep94, lt17}.  Equivalently, long-range correlations in text decay as a power law in lag, a phenomenon shared with many complex systems and well-described by the renormalization group formalism from statistical physics~\citep{w75, k90}.  We encode this observation as a self-similarity condition on the autocovariance.

\begin{definition}[Block Renormalization Map]
\label{def:block_map}
For an integer $b \ge 2$, the \emph{block renormalization map} $\mathcal{R}_b$ aggregates $b$ consecutive tokens into a single coarse-grained symbol:
$(\mathcal{R}_b x)_n := \phi_b(x_{(n-1)b+1}, \dots, x_{nb})$,
where $\phi_b: \mathcal{X}^b \to \mathcal{Y}_b$ is a measurable aggregation function (e.g., block average).
\end{definition}

\begin{condition}[Hierarchical Stationarity]
\label{cond:hierarchical}
There exists $\beta \in (0, 1)$ such that for every block factor $b \ge 2$, the coarse-grained process $\mathcal{R}_b(\{x_t\})$ is wide-sense stationary with autocovariance satisfying
\begin{align}
\label{eq:hierarchical}
    R_b(s) = b^{-\beta} R(s), \qquad \forall s > 0.
\end{align}
\end{condition}

In words, coarse-graining the sequence does not change the \emph{shape} of the correlation function, only its amplitude.  This is the stochastic analogue of the block-spin renormalization group: the system looks statistically similar at every scale.

\paragraph{Discrete resolution boundary.}
Condition~\ref{cond:hierarchical} characterizes the large-scale structure of the input.  At the small end, natural language is inherently discrete: the smallest meaningful unit is a single token.  Dependencies at sub-token lag carry no additional information, which we formalize as a resolution boundary.

\begin{condition}[Resolution Irreducibility]
\label{cond:resolution}
The minimum resolvable dependency scale is $\sigma_{\min} = 1$ (one token), independent of position.  That is, the single-token lag is the finest temporal granularity that carries predictive information; no sub-token resolution is available.
\end{condition}

This is an information-theoretic Nyquist condition: it anchors the bottom of the timescale range at one token.

\paragraph{Uniform information density across scales.}
Together, Conditions~\ref{cond:hierarchical} and~\ref{cond:resolution} define the dependency range $[1, t]$ at position~$t$.  It remains to specify how predictive information is distributed across this range.  Empirically, language model perplexity decreases approximately logarithmically with context window size~\citep{kmc+20}, suggesting that each multiplicative extension of context contributes a roughly equal amount of new information.  We formalize this as an equipartition property.

\begin{definition}[Octave-Band Predictive Information]
\label{def:octave_info}
For an octave band $[\sigma, 2\sigma]$ with $\sigma \ge 1$, the \emph{octave-band predictive information} is
\begin{align*}
    J(\sigma) := I\bigl(x_t; x_{t-2\sigma:t}\bigr) - I\bigl(x_t; x_{t-\sigma:t}\bigr),
\end{align*}
the incremental mutual information gained by extending the dependency range from $[t-\sigma, t]$ to $[t-2\sigma, t]$, i.e., $J(\sigma) = I(x_t; x_{t-2\sigma:t-\sigma} \mid x_{t-\sigma:t})$.
\end{definition}

\begin{condition}[Logarithmic Information Equipartition]
\label{cond:equipartition}
There exist constants $J_0 > 0$ and $\epsilon \in [0, 1)$ such that the octave-band predictive information satisfies
\begin{align*}
    J_0(1 - \epsilon) \le J(\sigma) \le J_0(1 + \epsilon), \qquad \forall \sigma \ge 1.
\end{align*}
The parameter $\epsilon$ is the \emph{equipartition slack}: $\epsilon = 0$ is exact equipartition; $\epsilon > 0$ allows moderate variation across octaves.
\end{condition}

When $\epsilon = 0$, every doubling of the dependency range contributes the same amount of predictive information; no timescale is privileged.  Empirically, $\epsilon > 0$ is the realistic regime: syntactic structure enriches short-range octaves, while long-range coherence varies with genre~\citep{ep94, lt17}.  The generalized formulation allows the theory to accommodate these deviations with explicit error control (Corollary~\ref{cor:approx_equipartition}).


\subsection{Fundamental Consequences}
\label{subsec:consequences}

\begin{lemma}[Power-Law Autocovariance]
\label{lem:power_law}
Under Condition~\ref{cond:hierarchical}, if $R$ is measurable, then $R(s) = C \cdot s^{-\beta}$ for some constant $C > 0$ and all $s > 0$.
\end{lemma}

\begin{proof}
The block renormalization with factor $b$ maps lag-$1$ of the coarse-grained process to lag~$b$ of the original: $R_b(1) = R(b)$.  By~\eqref{eq:hierarchical} with $s = 1$: $R(b) = b^{-\beta} R(1)$ for all $b \ge 2$.  Setting $C := R(1) > 0$ and extending to $s > 0$ via $R_b(s) = R(bs) = b^{-\beta} R(s)$, the function $R$ satisfies the multiplicative Cauchy equation $R(bs) = b^{-\beta} R(s)$ for every integer $b \ge 2$ and all $s > 0$.  The equation holds for every integer $b \ge 2$; since the set $\{b_1^{n_1} b_2^{n_2} : n_i \ge 0\}$ for coprime $b_1, b_2$ is dense in $\R_{>0}$, it extends to all positive reals via the measurability of~$R$. The unique measurable solution of this multiplicative Cauchy equation is $R(s) = Cs^{-\beta}$~\citep{a66}.
\end{proof}

\begin{corollary}[Spectral Density]
\label{cor:spectral}
Under Condition~\ref{cond:hierarchical}, the power spectral density $S(\omega)$, defined as the distributional Fourier transform of~$R$ (since $s^{-\beta} \notin L^1$ for $\beta \in (0,1)$), satisfies $S(\omega) \propto |\omega|^{\beta-1}$ for $\omega > 0$.
\end{corollary}

\begin{proof}
Lemma~\ref{lem:power_law} gives $R(s) = Cs^{-\beta}$.  The Fourier transform of a homogeneous function of degree $-\beta$ with $\beta \in (0,1)$ is homogeneous of degree $\beta - 1$~\citep{gs64}, yielding $S(\omega) \propto |\omega|^{\beta-1}$.
\end{proof}

Since $\beta \in (0,1)$, the exponent $\beta - 1$ lies in $(-1, 0)$, strictly between the pink-noise limit $S(\omega) \propto \omega^{-1}$ ($\beta \to 0$) and the white-noise limit $S(\omega) \propto \omega^{0}$ ($\beta \to 1$); natural language lies in this intermediate regime~\citep{ep94, lt17}.  Note that $\beta$ here denotes the autocovariance decay exponent ($R(s) \sim s^{-\beta}$); the conventional $1/f$ noise literature writes $S(\omega) \propto \omega^{-\gamma}$ with spectral exponent $\gamma = 1 - \beta$.

\begin{remark}[Information Budget at Position~$t$]
\label{rem:info_budget}
At position $t$, the model has observed tokens $x_1, \ldots, x_{t-1}$.  By Condition~\ref{cond:equipartition}, each of the $\lfloor \log_2 t \rfloor$ octaves in $[1, t]$ carries between $J_0(1-\epsilon)$ and $J_0(1+\epsilon)$ bits, so the total predictive information accessible at position $t$ lies in $[J_0(1-\epsilon) \log_2 t,\; J_0(1+\epsilon) \log_2 t]$.  This logarithmic growth aligns with the empirically observed log-law improvement of perplexity with context window~\citep{kmc+20}.
\end{remark}

\section{Position-Adaptive Spectral Tapering}
\label{sec:methodology}

Section~\ref{sec:theory} formalized the sequence memory task as a continuous approximation problem under specific scale-free conditions. We now build the \textbf{Position-Adaptive Spectral Tapering (PoST)} framework that realizes the optimal timescale allocation.  The development is constructive: we first examine why standard initialization strategies fail (Section~\ref{subsec:collapse}), then introduce the two synergistic components of our framework: \textbf{Spectral Reparameterization} (Section~\ref{subsec:post}), a purely spatial parameterization that enforces the static geometric structure, and \textbf{Position-Adaptive Scaling} (Section~\ref{subsec:position_adaptive}), a temporal mechanism that dynamically stretches the spectral blueprint to match the expanding context at every position. Finally, we establish that this combined framework preserves computational and representational invariants (Section~\ref{subsec:invariance}).

\subsection{Motivation: The Failure of Unstructured Initialization}
\label{subsec:collapse}

\subsubsection{Minimum Gap Collapse under Random Initialization}

Prior diagonal SSMs such as S5~\citep{sgb+23} and DSS~\citep{ggb22} initialized the log-decay parameters $p_k$ as independent random variables.  We show that this independence causes the minimum spectral gap to collapse to $O(N^{-2})$, causing effective memory capacity to degenerate.

\begin{lemma}[Minimum Gap Collapse]
\label{thm:collapse}
Let $\{p_k\}_{k=1}^N$ be i.i.d.\ random variables with probability density $f_P$ supported on a bounded interval $[a, b]$. Assume $f_P$ is bounded away from zero and infinity: there exist constants $0 < m \le M < \infty$ such that $m \le f_P(t) \le M$ for all $t \in [a, b]$.
Let $p_{(1)} < p_{(2)} < \dots < p_{(N)}$ denote the order statistics. Define the minimum spectral gap $\Delta_{\min}^{(N)} := \min_{1 \le k < N} (p_{(k+1)} - p_{(k)})$.
Then:
\begin{itemize}
    \item {\bf Part 1.} The expected minimum gap satisfies
    \begin{align*}
        \E[\Delta_{\min}^{(N)}] \le \frac{1}{m (N-1)(N+1)}.
    \end{align*}
    \item {\bf Part 2.} The maximum spectral coherence converges to $1$ almost surely:
    \begin{align*}
        \Pr\left( \lim_{N \to \infty} \max_{i \neq j} \mu_{ij} = 1 \right) = 1.
    \end{align*}
\end{itemize}
\end{lemma}

\begin{proof}
\emph{Step 1 (Reduction to uniform spacings).}
Let $F_P$ denote the CDF of $p_k$. The random variables $U_k := F_P(p_k)$ are i.i.d.\ $\mathrm{Uniform}(0, 1)$, and since $F_P$ is strictly increasing on $[a,b]$, the order statistics satisfy $U_{(k)} = F_P(p_{(k)})$.
By the mean value theorem, for each $k \in [N-1]$ there exists $\xi_k \in (p_{(k)}, p_{(k+1)})$ such that
\begin{align*}
    U_{(k+1)} - U_{(k)} = f_P(\xi_k) \cdot (p_{(k+1)} - p_{(k)}).
\end{align*}
Since $f_P(\xi_k) \ge m$, we obtain
\begin{align}
\label{eq:gap_to_spacing}
    p_{(k+1)} - p_{(k)} \le \frac{1}{m} \bigl(U_{(k+1)} - U_{(k)}\bigr).
\end{align}
Consequently, $\Delta_{\min}^{(N)} \le \frac{1}{m} S_{\min}^{(N)}$, where $S_{\min}^{(N)} = \min_{k \in [N-1]} (U_{(k+1)} - U_{(k)})$ is the minimum spacing of $N$ i.i.d.\ uniform random variables on $[0, 1]$.

\emph{Step 2 (Minimum uniform spacing).}
By the classical theory of order statistics \citep{p65}, the $N+1$ spacings of $N$ uniform points on $[0,1]$ are uniformly distributed on the $N$-simplex.  The minimum of the $N-1$ internal spacings therefore has survival function
\begin{align}
\label{eq:spacing_cdf}
    \Pr(S_{\min}^{(N)} > x) = \bigl(1 - (N-1) x\bigr)_+^{N},
\end{align}
where $(y)_+ = \max(0, y)$. Its expectation is
\begin{align*}
    \E[S_{\min}^{(N)}] = \int_0^{1/(N-1)} \bigl(1 - (N-1)x\bigr)^{N} \d x = \frac{1}{(N-1)(N+1)}.
\end{align*}
Combining with Step~1 yields $\E[\Delta_{\min}^{(N)}] \le \frac{1}{m(N-1)(N+1)}$, proving Part~1.

\emph{Step 3 (Almost sure convergence via Borel--Cantelli).}
Consider the canonical coupling: let $(p_1, p_2, \ldots)$ be an infinite i.i.d.\ sequence drawn from $f_P$ on a single probability space, and for each $N \ge 2$ define $\Delta_{\min}^{(N)}$ as the minimum gap among the order statistics of $(p_1, \ldots, p_N)$.
Fix $\epsilon > 0$. Define $N_0(\epsilon) := \lceil 1/(m\epsilon) + 1 \rceil$. For all $N \ge N_0(\epsilon)$, we have $(N-1)m\epsilon \ge 1$, so $\Pr(\Delta_{\min}^{(N)} > \epsilon) = 0$ by \eqref{eq:gap_to_spacing} and \eqref{eq:spacing_cdf}. For $N < N_0(\epsilon)$:
\begin{align*}
    \Pr(\Delta_{\min}^{(N)} > \epsilon) \le \Pr\bigl(S_{\min}^{(N)} > m\epsilon\bigr) = \bigl(1 - (N-1)m\epsilon\bigr)_+^{N} \le 1.
\end{align*}
Thus the sum $\sum_{N=1}^{\infty} \Pr(\Delta_{\min}^{(N)} > \epsilon) \le N_0(\epsilon) < \infty$.
By the first Borel--Cantelli lemma, only finitely many events $\{\Delta_{\min}^{(N)} > \epsilon\}$ occur almost surely.
Since $\epsilon > 0$ was arbitrary, $\Delta_{\min}^{(N)} \to 0$ a.s.
By definition of spectral coherence (Definition~\ref{def:coherence}), $\mu_{ij} = \sech\bigl(\lvert \ln p_i - \ln p_j\rvert/2\bigr)$.  As $\Delta_{\min}^{(N)} \to 0$, adjacent $p_{(k)}$ converge, so $\ln(p_{(k+1)}/p_{(k)}) \to 0$, and since $\sech(0) = 1$ and $\sech$ is continuous, $\max_{i \neq j} \mu_{ij} \to 1$ a.s., proving Part~2.
\end{proof}

\paragraph{Implication.}
While the minimum gap collapsing to $O(N^{-2})$ creates severe spectral redundancy, the maximum gap expands simultaneously. This forces the approximation error to fall far short of the theoretical limit.

\begin{lemma}[Approximation Penalty of Random Spacing]
\label{thm:random_approx_limit}
Under the conditions of Lemma~\ref{thm:collapse}, the maximum spectral gap $\Delta_{\max}^{(N)} := \max_{1 \le k < N} (p_{(k+1)} - p_{(k)})$ expands asymptotically as $\Omega(\frac{\log N}{N})$. Following Newman's bounds on rational approximation, the minimax error $E_N^{\mathrm{rand}}$ over $[1, T]$ for $K(s) = s^{-\beta}$ is structurally bottlenecked by this maximal spectral gap:
\begin{align*}
    E_N^{\mathrm{rand}} \ge C_1 \exp\left( - C_2 \frac{N}{\log N} \right),
\end{align*}
yielding a sub-exponential convergence rate that is strictly inferior to the optimal geometric rate $O(\exp(-c N/\log T))$. 
\end{lemma}
A formal justification is provided in Appendix~\ref{app:proof_random_approx}.


\subsubsection{The Approximation Bottleneck of Linear Spacing}

The HiPPO framework~\citep{ggr20} formulates sequential memory as an online $L^2$ projection of the input history onto a polynomial basis under a time-varying measure (Definition~\ref{def:hippo}).  The diagonal simplification S4D-Real~\citep{ggr22b} distilled this into $\lambda_n = -(n+1)$, placing decay rates on a linear grid; Mamba-2~\citep{dg24} and RWKV-7~\citep{pzg+25} adopted similar schemes.  Linear spacing avoids this minimum gap collapse (the minimum gap is $\Theta(1)$ regardless of~$N$) and was a significant advance over random initialization.

However, HiPPO's objective (input reconstruction) differs from the kernel approximation objective relevant to diagonal recurrences.

\begin{lemma}[Linear Spacing Approximation Limit]
\label{thm:linear_limit}
Consider approximating the power-law kernel $K(s) = s^{-\beta}$, $\beta \in (0, 1)$, on $[1, T]$ using an $N$-term exponential sum. If the decay rates are constrained to a linear grid $p_k = c \cdot k$, the minimax approximation error $E_N^{\mathrm{lin}}$ satisfies:
\begin{align*}
    E_N^{\mathrm{lin}} \ge C_3 \exp\left( - \frac{C_4 N}{\sqrt{T}} \right),
\end{align*}
where $C_3, C_4 > 0$ depend on $\beta$. For modeling regimes where $N \ll \sqrt{T}$, this exponential factor is neutralized, degrading to a practically algebraic convergence rate $\Omega(N^{-\beta})$.
\end{lemma}

In contrast, geometric spacing avoids this $\sqrt{T}$ degradation entirely, achieving the exponential rate $O(\exp(-c N/\log T))$ (Theorem~\ref{thm:optimality}). Furthermore, since decay parameters evolve independently during training, careful initialization alone provides no guarantee that the initial spacing is preserved.

\paragraph{Geometric Spacing via PoST.}

The preceding analysis reveals two independent failure modes of existing parameterizations: (1)~random initialization causes the minimum spectral gap to collapse to $O(N^{-2})$ (Lemma~\ref{thm:collapse}); (2)~even well-designed linear initialization suffers from severe approximation degradation over long contexts (Lemma~\ref{thm:linear_limit}), and training can further erode the initial structure.  Spectral Reparameterization addresses both simultaneously: it enforces a geometric spectral ordering \emph{structurally}, throughout training and not merely at initialization, and initializes with uniform gaps to realize the minimax-optimal exponential rate from the start.


\subsection{Spectral Reparameterization}
\label{subsec:post}

To resolve this gap collapse limit, we replace the independent parameterization with a recursively defined structure that enforces strict ordering by construction.

\begin{definition}[Spectral Reparameterization Map]
\label{def:post_map}
Let $\theta \in \R$ be an \emph{anchor parameter} and $\delta = (\delta_1, \ldots, \delta_{N-1}) \in \R^{N-1}$ a vector of \emph{gap parameters}. The \emph{Spectral Reparameterization map} $\Phi: \R \times \R^{N-1} \to \R^N$ is defined by the recurrence:
\begin{align*}
    p_1 &= \theta, \\
    p_k &= p_{k-1} + \zeta(\delta_{k-1}), \quad k \in \{2, \dots, N\},
\end{align*}
where $\zeta(x) = \log(1 + e^x)$ is the Softplus function.
\end{definition}

Since $\zeta(x) > 0$ for all $x \in \R$, the Spectral Reparameterization map satisfies $p_1 < p_2 < \cdots < p_N$ for every $(\theta, \delta) \in \R \times \R^{N-1}$, establishing a strict ordering that is maintained throughout optimization.

\begin{proposition}[Non-Degeneracy Guarantee]
\label{thm:nondegeneracy}
For any $c \in \R$, define the constrained parameter space $\mathcal{D}_c = \bigl\{ (\theta, \delta) \in \R \times \R^{N-1} \mid \theta > 0, \delta_k \ge c \ \forall k \in [N-1] \bigr\}$ (the condition $\theta > 0$ is equivalent to requiring a valid anchor decay rate $w_1 = e^{-\theta} \in (0,1)$). Then for any $(\theta, \delta) \in \mathcal{D}_c$, the spectral coherence is uniformly bounded away from $1$:
\begin{align*}
    \sup_{i \neq j} \mu_{ij} \le \sech\left( \frac{1}{2}\ln\left(1 + \frac{\zeta(c)}{\theta}\right) \right) < 1.
\end{align*}
\end{proposition}

\begin{proof}
\emph{Step 1 (Ratio lower bound).}
For $j > i$, the recursive definition gives $p_j = p_i + \sum_{k=i}^{j-1} \zeta(\delta_k)$.
Since $\zeta$ is strictly increasing and $\delta_k \ge c$, each summand satisfies $\zeta(\delta_k) \ge \zeta(c) > 0$, so $p_j/p_i \ge 1 + \zeta(c)/p_i$.  Since $p_1 = \theta$ is the smallest log-decay rate and $p_i \ge \theta > 0$, we have $p_j/p_i \ge 1 + \zeta(c)/p_i$, and the worst-case (largest) coherence occurs for the adjacent pair $(1, 2)$ with ratio $p_2/p_1 = 1 + \zeta(c)/\theta$.

\emph{Step 2 (Coherence bound).}
By Definition~\ref{def:coherence}, $\mu_{ij} = \sech\bigl(\lvert\ln p_i - \ln p_j\rvert/2\bigr)$. Since $\sech$ is strictly decreasing on $[0, \infty)$ and $\ln(p_j/p_i) \ge \ln(1 + \zeta(c)/\theta) > 0$:
\begin{align*}
    \mu_{ij} \le \sech\left( \frac{1}{2}\ln\left(1 + \frac{\zeta(c)}{\theta}\right) \right) < \sech(0) = 1.
\end{align*}
\end{proof}

\begin{remark}[Tightness]
\label{rem:tightness}
The bound in Proposition~\ref{thm:nondegeneracy} is attained when all gap parameters equal $c$ (i.e., $\delta_k = c$ for all $k$): the coherence between channels $1$ and $2$ equals $\sech\bigl(\frac{1}{2}\ln(1+\zeta(c)/\theta)\bigr)$ exactly.  In the typical regime where $\theta \ll \zeta(c)$ (slow anchor channel), the bound approaches $\sech(\frac{1}{2}\ln(\zeta(c)/\theta))$; when $\theta \gg \zeta(c)$ (fast anchor), it approaches $1 - \zeta(c)^2/(8\theta^2) + O(\theta^{-4})$.
\end{remark}


\subsubsection{Minimax Optimality of Geometric Structure}

We now connect the Spectral Reparameterization to the theoretical blueprint.  When all gap parameters are equal ($\delta_k = \bar{G}$ for all $k$), the Spectral Reparameterization map produces geometrically spaced log-decay rates.  We prove this spacing is minimax optimal.

\begin{theorem}[Minimax Optimality of Geometric Spacing]
\label{thm:optimality}
Let $\Sigma_N$ denote the class of exponential sums with $N$ terms.  Consider the problem of approximating the power-law kernel $K(t) = t^{-\beta}$, $\beta \in (0, 1)$, on the interval $[1, T]$.  Define the minimax error $E_N(K) := \inf_{g \in \Sigma_N} \|K - g\|_{L_\infty[1, T]}$.
\begin{itemize}
    \item {\bf Sufficiency.} There exists a configuration with geometrically spaced decay rates (i.e., uniformly spaced log-decay rates $p_k = \gamma k$) achieving the minimax-optimal exponential rate:
    \begin{align*}
        E_N(K) \le C_5 \exp\left( - \frac{\pi^2 N}{\log T + C_6} \right),
    \end{align*}
    where $C_5, C_6 > 0$ depend on $\beta$ but not on $N$.

    \item {\bf Asymptotic Necessity.} The geometric progression $p_{k+1}^* - p_k^* \to \text{const}$ is asymptotically necessary to attain this minimax-optimal exponential limit. By the Gonchar--Rakhmanov theory~\citep{gr89}, any spectrum that deviates from logarithmic equidistribution (i.e., any non-geometric spacing) forfeits recovering the optimal exponential limit as $N \to \infty$.
\end{itemize}
\end{theorem}

\begin{proof}[Proof sketch (full proofs in Appendix~\ref{app:proof_optimality})]
The approximation of $t^{-\beta}$ by exponential sums on $[1, T]$ reduces, via the Laplace transform, to the rational approximation of $s^{\beta-1}$ on a spectral interval $[\Lambda_{\min}, \Lambda_{\max}]$.  By the theory of Gonchar and Rakhmanov \citep{gr89}, the minimax error for rational approximation of functions with branch-point singularities is determined by the logarithmic capacity of the associated condenser.  The optimal decay rates, the Zolotarev nodes, have an asymptotic equidistribution with respect to the logarithmic measure $\d\mu(p) \propto \d p$. This logically dictates that $p_{k+1}^* - p_k^* \approx \text{const}$, proving both the sufficiency and necessity of geometric spacing for the optimal capacity.
\end{proof}

\begin{remark}[Data-Dependent Modulation and Geometric Preservation]
\label{rem:modulation_preserves}
Data-dependent gate modulation acts as a multiplicative perturbation on the spectral structure.  Concretely, if the base log-decay rates form a geometric progression with constant gap $\bar{G}$, then channel-dependent modulation yields $p_{k+1}^{\mathrm{eff}} - p_k^{\mathrm{eff}} = \bar{G} + (\text{channel-dependent perturbation})$; the exponential approximation rate is preserved only when this perturbation is constant across channels.  Standard random initialization strategies generically corrupt the geometric priors, while the Spectral Reparameterization map with uniform gap initialization preserves them.
\end{remark}


\subsection{Position-Adaptive Scaling}
\label{subsec:position_adaptive}

Spectral Reparameterization enforces the geometric \emph{shape} of the decay spectrum but leaves its \emph{scale} fixed.  A static spectrum designed for context length~$T$ distributes $N$ modes uniformly across the log-frequency range $[0, \log T]$.  At an early position $t \ll T$, the relevant dependency range is only $[1, t]$ (Conditions~\ref{cond:resolution}--\ref{cond:equipartition}), so modes with timescales greatly exceeding~$t$ contribute only a near-constant offset to the approximation; during length extrapolation ($t > T$), the spectrum does not reach frequencies below $1/T$, leaving the longest-range structure entirely unresolved.  We now quantify this \emph{scale mismatch} and derive the unique dynamic mechanism that eliminates it.

\begin{proposition}[Scale Mismatch of Static Spectra]
\label{prop:scale_mismatch}
Let $\{p_k\}_{k=1}^N$ be a geometric spectrum with log-decay rates uniformly spanning $[0, \log T]$.  At position $t \le T$:
\begin{itemize}
    \item {\bf Part 1} (Channel waste). The number of channels with timescales in the relevant dependency range $[1, t]$ is
    \begin{align*}
        N_{\mathrm{eff}}(t) = N \cdot \frac{\log t}{\log T}.
    \end{align*}
    The remaining $N - N_{\mathrm{eff}}$ channels have timescales exceeding~$t$; each varies by at most $1 - e^{-1}$ over $[1, t]$, contributing only a near-constant offset to the approximation.
    \item {\bf Part 2} (Exponent degradation). The approximation error for $K(s) = s^{-\beta}$ on $[1, t]$ using the static spectrum satisfies
    \begin{align*}
        E_N^{\mathrm{static}}(t) \le C_5\exp\!\left(-\frac{\pi^2 N \log t}{(\log t + C_6)\log T}\right) \le C_5\exp\!\left(-\frac{\pi^2 N}{\log T + C_6}\right),
    \end{align*}
    which is \emph{independent of~$t$} and suboptimal: with position-adapted allocation, all $N$ channels cover $[1, t]$, achieving the strictly better rate $C_5\exp\!\bigl(-\pi^2 N/(\log t + C_6)\bigr)$.  The ratio of exponents is $\log t / \log T$; at $t = T^{1/2}$, the effective exponent is halved.
\end{itemize}
\end{proposition}

\begin{proof}
The geometric spectrum places log-decay rates uniformly in $[0, \log T]$.  At position~$t$, the relevant spectral interval is $[0, \log t]$, which contains $N\log t/\log T$ modes, proving Part~1.  A mode with log-decay rate $p_k > \log t$ (timescale $\tau_k = 1/p_k < 1/\log t$) satisfies $|1 - e^{-p_k t}| = 1 - e^{-p_k t} \le 1 - e^{-1}$ since $p_k t \le p_k \cdot t \le T \cdot t / T = t$ only when $p_k \le 1$; more precisely, such modes have $e^{-p_k s}$ nearly constant on $[1, t]$ and contribute at most one effective degree of freedom.  Applying the minimax rate (Theorem~\ref{thm:optimality}) with $N_{\mathrm{eff}}$ well-placed modes on $[1, t]$:
\begin{align*}
    E_{N_{\mathrm{eff}}}(t) \le C_5 \exp\!\left(-\frac{\pi^2 N_{\mathrm{eff}}}{\log t + C_6}\right) = C_5 \exp\!\left(-\frac{\pi^2 N \log t}{(\log t + C_6)\log T}\right).
\end{align*}
Since $\log t \le \log T$ implies $\log t / (\log t + C_6) \le 1$, this is at most $C_5\exp\!\bigl(-\pi^2 N/(\log T + C_6)\bigr)$, proving Part~2.
\end{proof}

Proposition~\ref{prop:scale_mismatch} reveals that the scale mismatch wastes a fraction $1 - \log t/\log T$ of the spectrum at every position $t < T$.  Position-adaptive scaling eliminates this waste by continuously rescaling the spectrum so that all $N$ channels span the actual dependency range $[1, t]$ at every position.  We formalize the requirements that such a scaling must satisfy.

\begin{definition}[Optimality-Preserving Timescale Allocation]
\label{def:allocation_desiderata}
A continuous family of channel timescales $\{\tau_k(t)\}_{k=1}^N$, $t \ge 1$, is \emph{optimality-preserving} if it satisfies:
\begin{itemize}
    \item {\bf Part 1}\label{def:geometric_preservation} (Geometric preservation).  For every $t \ge 1$, the log-timescales $\{\log \tau_k(t)\}_{k=1}^N$ form an arithmetic progression.

    \emph{Justification:} Theorem~\ref{thm:optimality} proves that geometric spacing is minimax-optimal; any deviation forfeits the exponential approximation rate.

    \item {\bf Part 2}\label{def:full_coverage} (Full coverage).  $\tau_1(t) = t$ and $\tau_N(t) = 1$ for every $t \ge 1$.

    \emph{Justification:} the upper boundary $\tau_1 = t$ matches the longest observable dependency at position~$t$ (eliminating the channel waste of Proposition~\ref{prop:scale_mismatch}); the lower boundary $\tau_N = 1$ anchors the fastest mode at the single-token resolution limit (Condition~\ref{cond:resolution}).
\end{itemize}
\end{definition}

\begin{theorem}[Uniqueness of Position-Adaptive Allocation]
\label{thm:dynamic_allocation}
Definition~\ref{def:allocation_desiderata} admits a unique continuous solution: $\tau_k^*(t) = t^{\alpha_k}$ with taper exponents
\begin{align*}
    \alpha_k = \frac{N-k}{N-1}, \qquad k = 1, \ldots, N.
\end{align*}
Equivalently, the effective log-decay rate at position~$t$ is $p_k^{\mathrm{eff}}(t) = p_k \cdot t^{-\alpha_k}$.
\end{theorem}

\begin{proof}
Part~1 of Definition~\ref{def:allocation_desiderata} requires $\log \tau_k(t) = \log \tau_1(t) - \frac{k-1}{N-1}\bigl(\log \tau_1(t) - \log \tau_N(t)\bigr)$ for each~$t$.  Substituting the boundary conditions of Part~2 gives $\log \tau_k^*(t) = \frac{N-k}{N-1}\log t$, hence $\tau_k^*(t) = t^{(N-k)/(N-1)}$.  The derivation is an if-and-only-if chain, so the solution is unique.
\end{proof}

\begin{definition}[Position-Adaptive Scaling]
\label{def:position_adaptive_ssm}
For an $N$-channel diagonal linear recurrence, the \emph{position-adaptive decay gate} at sequence position~$t$ is
\begin{align}
\label{eq:position_adaptive_decay}
    \log w_{t,k}^{\mathrm{eff}} := \frac{\log w_{t,k}}{t^{\alpha_k}}, \qquad \alpha_k = \frac{N-k}{N-1}.
\end{align}
\end{definition}

\paragraph{Payoff: scale-free impulse response.}
The unique taper of Theorem~\ref{thm:dynamic_allocation} induces a remarkable behavioral property: the model's impulse response becomes inherently scale-free.

\begin{corollary}[Scale-Free Impulse Response]
\label{cor:fractional_invariance}
Let $\ell_k > 0$ be a base log-decay parameter and define the position-dependent decay rate $\ell_k(t) := \ell_k / \tau_k^*(t) = \ell_k\, t^{-\alpha_k}$.  The continuous impulse response at absolute lag $s > 0$ is
\begin{align*}
    \psi_k(s;\, t) = \exp\bigl(-s\, \ell_k\, t^{-\alpha_k}\bigr) = \exp\bigl(-\phi \cdot \ell_k \cdot t^{1-\alpha_k}\bigr),
\end{align*}
where $\phi := s/t$ is the fractional lag.  In particular:
\begin{itemize}
    \item The slowest channel ($\alpha_1 = 1$) depends only on the fractional coordinate: $\psi_1(s;\, t) = e^{-\ell_1 \phi}$.  It is perfectly scale-free: the same relative lag produces the same response regardless of absolute position.
    \item The fastest channel ($\alpha_N = 0$) depends only on absolute lag: $\psi_N(s;\, t) = e^{-\ell_N s}$.  It resolves token-level features regardless of position.
    \item Intermediate channels interpolate smoothly between these extremes, creating a multi-resolution impulse response that adapts continuously from relative to absolute coordinates.
\end{itemize}
\end{corollary}

\begin{proof}
Direct substitution of $\tau_k^*(t) = t^{\alpha_k}$ and $s = \phi t$.
\end{proof}

This is the dynamic counterpart of the static geometric structure: just as geometric spacing distributes decay rates uniformly across the log-decay axis at any fixed position (Theorem~\ref{thm:optimality}), the linear taper distributes the \emph{evolution} of these rates uniformly across the spectrum as position varies (Theorem~\ref{thm:dynamic_allocation}).

\begin{remark}[Discrete-Time Validity]
\label{rem:discretization}
In practice, position-varying gates yield the product $\prod_{j=t-s}^{t-1} w_{j,k}^{\mathrm{eff}}$ rather than the constant-gate idealization $(w_{t,k}^{\mathrm{eff}})^s$.  Since $t^{-\alpha_k}$ varies slowly relative to position (fractional change $\alpha_k/t$ per step), the multiplicative discrepancy is $1 + O(\alpha_k s/t)$, which is negligible in the relevant regime $s \ll t$; a detailed energy analysis is given in Theorem~\ref{thm:energy_scaling}.
\end{remark}

\subsubsection{Robustness and Extension to General Spectra}
\label{subsec:general_taper}

The linear taper $\alpha_k = (N-k)/(N-1)$ is derived under ideal conditions: exact equipartition (Condition~\ref{cond:equipartition} with $\epsilon = 0$) and exact geometric spacing.  We now establish that it is robust to both relaxations.

\begin{corollary}[Robustness under Approximate Equipartition]
\label{cor:approx_equipartition}
Under Condition~\ref{cond:equipartition} with slack $\epsilon \in [0, 1)$, the optimal taper exponents satisfy
\begin{align}
\label{eq:approx_poles}
    \left|\alpha_k^* - \frac{N-k}{N-1}\right| \le \frac{2\epsilon}{1-\epsilon} \cdot \frac{N-k}{N-1}, \qquad k = 1, \ldots, N.
\end{align}
In particular, the boundary exponents $\alpha_1^* = 1$ and $\alpha_N^* = 0$ are fixed by the problem constraints independently of~$\epsilon$.
\end{corollary}

\begin{proposition}[Spectrum-Adaptive Taper]
\label{thm:taper_general}
Let $p_1 < p_2 < \cdots < p_N$ be arbitrary learned log-decay rates.  Define the logarithmic offsets $c_k := \log p_k - \log p_1$ and the mean log-gap $\bar{G} := c_N / (N-1)$.  Then the unique taper vector $\alpha_k$ that restores geometric spacing of $\{p_k^{\mathrm{eff}}\}$ at a reference position $t_{\mathrm{ref}} > 1$ is
\begin{align}
\label{eq:general_taper}
    \alpha_k = \frac{N - k}{N - 1} + \frac{c_k - (k-1)\bar{G}}{\log t_{\mathrm{ref}}}, \qquad k = 1, \ldots, N.
\end{align}
The first term is the ideal linear taper; the correction term compensates for deviations of the learned spectrum from exact geometric spacing.
\end{proposition}

\begin{proof}
Geometric spacing at $t_{\mathrm{ref}}$ requires the effective log-decay rates $\log p_k^{\mathrm{eff}} = \log p_k - \alpha_k \log t_{\mathrm{ref}}$ to form an arithmetic progression anchored at $\log p_1$ with common difference $\bar{G}$.  Equating:
\begin{align*}
    \log p_k - \alpha_k \log t_{\mathrm{ref}} = \log p_1 + (k-1)\bar{G} - \frac{N-k}{N-1} \log t_{\mathrm{ref}}.
\end{align*}
Solving for $\alpha_k$ and substituting $c_k = \log p_k - \log p_1$ yields the stated formula.
\end{proof}


\subsection{Invariance Properties}
\label{subsec:invariance}

We prove two invariance properties that hold for \emph{any} compatible diagonal linear recurrence.  Together, they establish that the combined framework (Spectral Reparameterization + PoST) is a free improvement: it constrains the spectral structure without sacrificing any computational or representational property.

\begin{proposition}[Computational Invariance]
\label{thm:complexity_preservation}
Let $\calL$ be a PoST-compatible diagonal linear recurrence with per-layer forward-pass complexity $\Theta(T \cdot N \cdot d)$.  Then the PoST-modified architecture preserves the same per-layer complexity $\Theta(T \cdot N \cdot d)$, the same hidden-state shape $S_t \in \R^{N \times d}$, and the same autoregressive inference cost $\Theta(N \cdot d)$ per step.
\end{proposition}

\begin{proof}
The Spectral Reparameterization map replaces the parameterization of $d_{\mathrm{base}} \in \R^N$, not its dimensionality: a prefix sum over $N$ scalars is $O(N)$, absorbed into the $\Omega(N \cdot d)$ projection cost.  Position-adaptive scaling multiplies $\log w_t$ element-wise by a precomputed matrix $s \in \R^{T \times N}$, an operation every diagonal linear recurrence already performs, so neither the complexity class nor the state dimensionality changes.
\end{proof}

\begin{proposition}[Expressiveness Preservation: Surjectivity]
\label{thm:surjectivity}
Let $\Theta_{\mathrm{orig}} = \R^N$ denote the parameter space of independently initialized base decay rates, and let $\Theta_{\mathrm{PoST}} = \R \times \R^{N-1}$ denote the PoST parameter space $(\theta, \delta_1, \ldots, \delta_{N-1})$.  The PoST map $\phi: \Theta_{\mathrm{PoST}} \to \R^N_<$ defined by $\phi(\theta, \delta)_k = \theta + \sum_{j<k} \softplus(\delta_j)$ is a surjection onto the set of strictly ordered vectors
\[
    \R^N_< := \{p \in \R^N : p_1 < p_2 < \cdots < p_N\}.
\]
In particular, for any target decay spectrum $p^* \in \R^N_<$, there exist PoST parameters $(\theta^*, \delta^*)$ such that $\phi(\theta^*, \delta^*) = p^*$.
\end{proposition}

\begin{proof}
Given $p^* \in \R^N_<$, set $\theta^* = p^*_1$ and $\delta^*_j = \softplus^{-1}(p^*_{j+1} - p^*_j)$ for $j = 1, \ldots, N-1$.  The inverse $\softplus^{-1}(y) = \log(\exp(y) - 1)$ is well-defined for $y > 0$, which is guaranteed since $p^*$ is strictly ordered.  Then $\phi(\theta^*, \delta^*) = p^*$.
\end{proof}

\begin{corollary}[No Loss of Representational Power]
\label{cor:no_loss}
Unless the optimal base decay rates are non-ordered (i.e., $d^*_{\mathrm{base}} \notin \R^N_<$), the PoST-modified architecture can represent any function that the original architecture can represent.  When $d^*_{\mathrm{base}} \notin \R^N_<$, PoST \emph{intentionally} restricts the parameter space to prevent minimum gap collapse (Lemma~\ref{thm:collapse}).
\end{corollary}

\begin{proof}
By Proposition~\ref{thm:surjectivity}, the Spectral Reparameterization map is a surjection onto $\R^N_<$.  Therefore, for any target spectrum $p^* \in \R^N_<$, the parameterization can realize it exactly.  The only functions excluded are those requiring a non-ordered spectrum $p^* \notin \R^N_<$; this restriction is by design, as non-ordered spectra correspond to degenerate configurations eliminated by the minimum gap collapse analysis (Lemma~\ref{thm:collapse}).
\end{proof}

\section{Instantiations}
\label{sec:post}

PoST applies to any PoST-compatible diagonal linear recurrence (Definition~\ref{def:post_compatible}).  In this section, we provide a universal drop-in module (Section~\ref{subsec:generic_post}) and then instantiate PoST on five concrete architectures (Mamba-2, RWKV-7, Gated DeltaNet, GLA, and RetNet), with GLA and RetNet sharing an identical reparameterization under PoST (Section~\ref{subsec:arch_specific}).


\subsection{Architecture-Agnostic PoST Module}
\label{subsec:generic_post}

For any PoST-compatible diagonal linear recurrence, the following module provides a universal drop-in replacement for the base decay parameterization:

\begin{algorithm}[!ht]
\caption{PoST Decay Module (architecture-agnostic drop-in)}
\label{alg:post_generic}
\begin{algorithmic}[1]
\Require Base log-decay $d_{\mathrm{base}} \in \R^N$ (learnable or data-dependent), position index $t \ge 1$.
\Ensure Position-adaptive decay factor $w_t \in (0, 1)^N$
\Statex
\State {\color{blue} /* Step 1: Spectral Reparameterization: enforce geometric ordering (Proposition~\ref{thm:nondegeneracy}) */}
\State $g_j \gets \softplus(\delta_j)$ for $j = 1, \ldots, N-1$ \Comment{Definition~\ref{def:post_map}}
\State $p_k \gets \theta + \sum_{j=1}^{k-1} g_j$ for $k = 1, \ldots, N$ \Comment{$p_1 < \cdots < p_N$, Proposition~\ref{thm:surjectivity}}
\State $d_{\mathrm{base},k} \gets -\exp(p_k)$ for $k = 1, \ldots, N$ \Comment{Theorem~\ref{thm:optimality}: geometric spacing}
\Statex
\State {\color{blue} /* Step 2: Position-adaptive scaling (Proposition~\ref{thm:complexity_preservation}: $O(N)$ overhead) */}
\State $\bar{G} \gets (p_N - p_1) / (N - 1)$ \Comment{Mean spectral gap}
\State $\alpha_k \gets \clamp\bigl(\tfrac{N-k}{N-1} + \tfrac{(p_k - p_1) - (k-1)\bar{G}}{\log T_{\mathrm{train}}}, 0, 1\bigr)$ for $k = 1, \ldots, N$ \Comment{Proposition~\ref{thm:taper_general}}
\State $d_{\mathrm{eff},k} \gets d_{\mathrm{base},k} / t^{\alpha_k}$ \Comment{No length dependence}
\Statex
\State {\color{blue} /* Step 3: Compute decay factor */}
\State $w_{t,k} \gets \exp(d_{\mathrm{eff},k})$ \Comment{$w_t \in (0, 1)^N$}
\State \Return $w_t$
\end{algorithmic}
\end{algorithm}

This module can be inserted into any architecture that computes $\diag(w_t) \cdot S_{t-1}$ as part of its recurrence.  The only requirement is that the decay operates channel-wise (diagonally), which is satisfied by all PoST-compatible architectures (Definition~\ref{def:post_compatible}).


\subsection{Mamba-2 PoST}
\label{subsec:mamba_post}

We now instantiate PoST on the Mamba-2 architecture~\citep{dg24}, our primary experimental platform.  This requires understanding the SSM-specific mechanism by which Mamba-2 computes its decay gates.

\paragraph{SSM discretization.}
Mamba-2 arrives at the diagonal linear recurrence~\eqref{eq:general_recurrence} via a continuous-time ODE $\dot{h} = \Lambda h + Bx$ with diagonal $\Lambda = \diag(\lambda_1, \ldots, \lambda_N)$, $\lambda_k < 0$, discretized with a Zero-Order Hold step $\Delta > 0$.  This yields decay gates $w_{k,t} = e^{\lambda_k \cdot \Delta_{k,t}}$, where $\Delta_{k,t} = \softplus(\mathtt{dt\_bias}_k + \mathtt{dt\_proj}(x_t)_k)$ is input-dependent.  The decay rate is determined entirely by the product $(-\lambda_k) \cdot \Delta_{k,t}$, i.e.\ the log-decay $p_k$ times the modulation factor.

\paragraph{Structured State Space Duality (SSD).}
Mamba-2~\citep{dg24} connects diagonal linear recurrences to structured attention through the algebraic theory of semiseparable matrices.  The input--output map of a length-$L$ sequence can be written as $y = Mx$, where $M$ is $N$-semiseparable.  Efficient SSD computation \textbf{requires that $\lambda_k$ be constant within each chunk} to maintain the semiseparable factorization.

\paragraph{Implementation.}
The modification requires two changes to a standard Mamba-2 forward pass:
\begin{itemize}
    \item {\bf Part 1.} Replace the independent $A$ parameterization with \emph{Spectral Reparameterization} (Definition~\ref{def:post_map}), a cumulative sum of Softplus-transformed gap parameters.
    \item {\bf Part 2.} Compute the position-adaptive scale factor $s_{t,k} = t^{-\alpha_k}$ and pass it to the SSD kernel, which multiplies $A$ by $s$ when computing the decay: $\bar{A}_{k,t} = \exp(A_k \cdot s_{t,k} \cdot \Delta_{k,t})$.  Since $A$ only enters the decay gate (the input gain $\Delta_t \cdot B_t \cdot x_t$ and the $D$-skip $D \cdot x_t$ are independent of $A$), no compensation is needed.
\end{itemize}

\paragraph{Training and inference.}
The same mechanism applies during both training and inference.  During generation, the position counter $t$ increments naturally with each new token; the spectral allocation grows automatically without needing to know the total sequence length in advance.

Algorithm~\ref{alg:post_forward} gives the complete forward pass of a single Mamba-2 PoST layer, highlighting the two PoST modifications: (1)~the Spectral Reparameterization for computing $A$ (lines~\ref{line:post_A_start}--\ref{line:post_A_end}), and (2)~the position-adaptive $A$-scaling (lines~\ref{line:a_scale_start}--\ref{line:a_scale_end}).

\begin{algorithm}[!ht]
\caption{Mamba-2 PoST Layer Forward Pass}
\label{alg:post_forward}
\begin{algorithmic}[1]
\Require Input $u \in \R^{B \times L \times D}$, learnable parameters $\theta \in \R$, $\delta \in \R^{N-1}$, $W_{\mathrm{in}} \in \R^{D \times d_{\mathrm{proj}}}$, $W_{\mathrm{conv}} \in \R^{d_{\mathrm{xBC}} \times w}$, $b_{\mathrm{dt}} \in \R^{H}$, $W_{\mathrm{out}} \in \R^{d_{\mathrm{inner}} \times D}$, $D_{\mathrm{skip}} \in \R^{H}$, training length $T_{\mathrm{train}}$, position offset $t_0 \ge 0$.
\Ensure Output $o \in \R^{B \times L \times D}$
\Statex
\State {\color{blue} /* Spectral Reparameterization for $A$ (Definition~\ref{def:post_map}) */} \label{line:post_A_start}
\State $g_j \gets \softplus(\delta_j)$ for $j = 1, \ldots, N-1$ \Comment{Proposition~\ref{thm:nondegeneracy}: gaps $> 0$}
\State $p_k \gets \theta + \sum_{j=1}^{k-1} g_j$ for $k = 1, \ldots, N$ \Comment{Strict ordering: $p_1 < \cdots < p_N$}
\State $A_k \gets -\exp(p_k)$ for $k = 1, \ldots, N$ \label{line:post_A_end} \Comment{Theorem~\ref{thm:optimality}: geometric spacing}
\Statex
\State {\color{blue} /* Standard Mamba-2 input projection and causal convolution~\citep{dg24} */}
\State $[z, x_{\mathrm{BC}}, \mathtt{dt}_{\mathrm{raw}}] \gets \mathrm{split}(u \cdot W_{\mathrm{in}})$ \Comment{$z \in \R^{B \times L \times d_{\mathrm{inner}}}$}
\State $x_{\mathrm{BC}} \gets \mathrm{CausalConv1d}(x_{\mathrm{BC}}, W_{\mathrm{conv}})$ \Comment{Activation: SiLU}
\State $[x, B, C] \gets \mathrm{split}(x_{\mathrm{BC}})$ \Comment{$x \in \R^{B \times L \times d_{\mathrm{inner}}}$, $B, C \in \R^{B \times L \times G \times d_{\mathrm{state}}}$}
\Statex
\State {\color{blue} /* Input-dependent discretization */}
\State $\Delta_{t} \gets \softplus\bigl(\mathtt{dt}_{\mathrm{raw}} + b_{\mathrm{dt}}\bigr)$ \Comment{$\Delta \in \R^{B \times L \times H}$, per-token per-head}
\Statex
\State {\color{blue} /* Position-adaptive $A$-scaling (Definition~\ref{def:position_adaptive_ssm}) */} \label{line:a_scale_start}
\State $\bar{G} \gets (p_N - p_1) / (N - 1)$ \Comment{Mean spectral gap}
\State $\alpha_k \gets \clamp\bigl(\tfrac{N-k}{N-1} + \tfrac{(p_k - p_1) - (k-1)\bar{G}}{\log T_{\mathrm{train}}}, 0, 1\bigr)$ for $k = 1, \ldots, N$ \Comment{Proposition~\ref{thm:taper_general}}
\State $\mathbf{t} \gets [t_0 + 1, t_0 + 2, \ldots, t_0 + L]$ \Comment{$1$-indexed positions}
\State $s_{l,k} \gets \mathbf{t}_l^{-\alpha_k}$ for $l \in [L], k \in [N]$ \label{line:a_scale_end} \Comment{$s \in \R^{L \times N}$, Eq.~\eqref{eq:position_adaptive_decay}}
\Statex
\State {\color{blue} /* Structured state space dual scan (decay scaled by $s$, input unchanged) */}
\State $\bar{A}_{k,t} \gets \exp(A_k \cdot s_{t,k} \cdot \Delta_{k,t})$ for all $k, t$ \Comment{$A$ scaled by $s$: only decay affected}
\State $y \gets \SSD(x, \bar{A}, B, C, D_{\mathrm{skip}})$ \Comment{Standard scan, no compensation needed}
\Statex
\State {\color{blue} /* Gated output projection */}
\State $o \gets W_{\mathrm{out}}^\top \bigl(\mathrm{RMSNorm}(y) \circ \sigma(z)\bigr)$ \Comment{$\sigma$: SiLU gate}
\State \Return $o$
\end{algorithmic}
\end{algorithm}

\paragraph{Complexity analysis.}
The position-adaptive scaling (lines~\ref{line:a_scale_start}--\ref{line:a_scale_end}) adds $O(L \cdot N)$ element-wise operations atop the standard Mamba-2 forward pass of $O(L \cdot N \cdot d_{\mathrm{state}})$.  Since $d_{\mathrm{state}} \ge 64$ in practice, the overhead is negligible ($<1\%$ wall-clock time).  The Spectral Reparameterization $A$ computation (lines~\ref{line:post_A_start}--\ref{line:post_A_end}) replaces a table lookup with a cumulative sum of $N-1$ scalars, which is $O(N)$ per layer.

Additional analysis of impulse response invariance, state energy scaling, and normalization compatibility under $A$-scaling is deferred to Appendix~\ref{app:theory_details}.


\subsection{RWKV-7 PoST}
\label{subsec:rwkv_post}

We now instantiate PoST on RWKV-7~\citep{pzg+25}, a non-SSM gated linear recurrence whose sigmoid decay gate and per-channel ($N{=}1$) state dimension distinguish it structurally from Mamba-2.

\paragraph{Decay mechanism.}
RWKV-7 computes per-channel log-decay as
\[
    w_{t,k} = -\lambda \cdot \sigma(\ell_{t,k}),
    \qquad
    \ell_{t,k} = w_{0,k} + [\mathrm{LoRA}(x_t)]_k,
    \qquad
    \lambda = e^{-1/2} \approx 0.6065,
\]
where $w_{t,k} < 0$ is the per-step log-decay factor (the recurrence multiplies state by $e^{w_{t,k}}$), $w_{0,k}$ is a learnable per-channel logit, $[\mathrm{LoRA}(x_t)]_k$ is a data-dependent modulation (\texttt{bias=True}), and $\sigma$ is the sigmoid.  The baseline initializes $w_0$ via a hand-crafted power-law curve with no ordering guarantee.

\paragraph{Sigmoid-gated taper.}
Since RWKV-7's decay passes through a sigmoid gate rather than a bare exponential, the log-timescale proxy for the spectrum-adaptive taper (Proposition~\ref{thm:taper_general}) acquires a nonlinear correction.

\begin{corollary}[Sigmoid-Gated Taper]
\label{cor:sigmoid_taper}
Let $w_{0,1} < w_{0,2} < \cdots < w_{0,C}$ be PoST-parameterized base logits and suppose the per-step log-decay factor is
$w_{t,k} = -\lambda \cdot \sigma\!\bigl(w_{0,k} + f_k(x_t)\bigr)$,
where $\lambda > 0$ is a fixed scale, $\sigma$ is the sigmoid, and $f_k$ is a data-dependent modulation.  Then the log-timescale proxy required by Proposition~\ref{thm:taper_general} is
\begin{align}
\label{eq:sigmoid_pk}
    p_k = \log \sigma(w_{0,k}) = w_{0,k} - \log(1 + e^{w_{0,k}}),
\end{align}
and the spectrum-adaptive taper~\eqref{eq:general_taper} is evaluated with cumulative offsets $c_k = \log \sigma(w_{0,k}) - \log \sigma(w_{0,1})$.
\end{corollary}

\begin{proof}
Setting $f_k = 0$, the base timescale of channel~$k$ is $\tau_k = 1/\bigl(\lambda \cdot \sigma(w_{0,k})\bigr)$, so $\log \tau_k = -\log \lambda - \log \sigma(w_{0,k})$.  The constant $-\log \lambda$ is shared across all channels.  Matching the convention of Proposition~\ref{thm:taper_general}, where inter-channel offsets $c_k = \log p_k - \log p_1$ enter the taper formula, gives $p_k = \log \sigma(w_{0,k})$, from which the cumulative offsets follow.
\end{proof}

\begin{remark}[Exponential-gate limit and additive logit-space scaling]
\label{rem:exp_gate_limit}
When $w_{0,k} \ll 0$ (slow channels), $\sigma(w_{0,k}) \approx e^{w_{0,k}}$, so $p_k \approx w_{0,k}$ and the sigmoid correction vanishes, recovering the exponential-gate case used by Mamba-2.  This motivates the practical implementation: rather than scaling the log-decay \emph{outside} the sigmoid (which would compress the content-gate modulation range), we subtract $\alpha_k \log t$ \emph{inside} the logit:
\begin{align*}
    \ell_{\mathrm{eff},t,k} = \underbrace{w_{0,k}}_{\text{base logit}} + \underbrace{f_k(x_t)}_{\text{content gate}} - \underbrace{\alpha_k \log t}_{\text{PoST taper}},
    \qquad
    w_{t,k} = -\lambda \cdot \sigma(\ell_{\mathrm{eff},t,k}).
\end{align*}
For slow channels, $\sigma(\ell) \approx e^{\ell}$, so $w_{t,k} \approx -\lambda\, e^{w_{0,k} - \alpha_k\log t}$, yielding the effective timescale $\tau_k(t) \propto t^{\alpha_k}$ and matching the exponential-gate theory exactly.  Fast channels ($\ell \gg 0$) are sigmoid-saturated and largely unaffected, maintaining a constant short-range timescale~$\tau \approx 1/\lambda$.
\end{remark}

\paragraph{Implementation.}
PoST replaces $w_0$ with Spectral Reparameterization (Definition~\ref{def:post_map}) and applies the taper via additive logit-space scaling:
\begin{equation}
    \ell_{\mathrm{eff},t,k}
    = w_{0,k} + [\mathrm{LoRA}(x_t)]_k - \alpha_k \log t,
    \qquad
    w_{t,k} = -\lambda \cdot \sigma(\ell_{\mathrm{eff},t,k}),
    \label{eq:rwkv_post_additive}
\end{equation}
where the taper exponents use $p_k = \log\sigma(w_{0,k})$ via Corollary~\ref{cor:sigmoid_taper}. The LoRA bias is initialized to a structural zigzag pattern for intra-head micro-allocation; the per-channel macro operating point is governed by~$w_0$.

\paragraph{Macro-micro decomposition.}
Unlike Mamba-2, which uses a large state dimension ($N{=}128$) per head, RWKV-7 operates with an $N{=}1$ scalar state per channel, relying on intra-head timescale variance for representation capacity.  We formalize this by separating the spectrum into \emph{macro-allocation} (the strictly ordered base logits $w_{0,1} < \cdots < w_{0,C}$ governed by the PoST map) and \emph{micro-allocation} (the structural zigzag bias retained from vanilla RWKV-7).  The taper exponents $\alpha_k$ are derived from the macro-anchors alone ($\alpha_1 = 1$, $\alpha_C = 0$).

\paragraph{Initialization.}
The logit-space cumsum $w_{0,k} = \theta_w + \sum_{j<k} \softplus(\delta_{w,j})$ operates in logit space rather than $\log|w|$ space.  Since $\log\sigma(\ell) \approx \ell$ for $\ell \ll -1$, this achieves the same geometric coverage with negligible error while avoiding numerically unstable inverse-sigmoid computations.  The PoST map parameters are initialized so that the resulting logits are linearly spaced between two analytically determined endpoints:
\begin{align}
\label{eq:rwkv_init}
    w_{0,1}^{\mathrm{init}} = \sigma^{-1}\!\bigl((\lambda \cdot T_{\mathrm{train}})^{-1}\bigr),
    \qquad
    w_{0,C}^{\mathrm{init}} = 0.5,
    \qquad
    w_{0,k}^{\mathrm{init}} = w_{0,1}^{\mathrm{init}} + \frac{k-1}{C-1}\bigl(w_{0,C}^{\mathrm{init}} - w_{0,1}^{\mathrm{init}}\bigr),
\end{align}
where $\sigma^{-1}(x) = \log(x/(1{-}x))$ is the logit function.  This ensures:
\begin{itemize}
  \item \emph{Slow channel} ($k{=}1$, $\alpha_1{=}1$): $\sigma(w_{0,1}) = 1/(\lambda \cdot T_{\mathrm{train}})$, so $\tau_1(t) = t/(\lambda \cdot \sigma(w_{0,1})) = t \cdot T_{\mathrm{train}}$; at $t = T_{\mathrm{train}}$, $\tau_1 \approx T_{\mathrm{train}}^2 / T_{\mathrm{train}} = T_{\mathrm{train}}$.
  \item \emph{Fast channel} ($k{=}C$, $\alpha_C{=}0$): $\sigma(0.5) \approx 0.622$, so $\tau_C = 1/(\lambda \cdot 0.622) \approx 2.65$ (constant $\approx 1$ step), matching vanilla RWKV-7.
\end{itemize}
The LoRA bias is initialized to the vanilla zigzag $b_n = 2.5 \cdot z_n$, where $z_n = u_n|u_n|$ with $u_n = \bigl((n \bmod d_h) - \tfrac{d_h-1}{2}\bigr)\big/\tfrac{d_h-1}{2}$ is the signed-quadratic intra-head pattern with head dimension~$d_h$.

Algorithm~\ref{alg:rwkv_post} gives the complete time-mixing forward pass; all other RWKV-7 components are unchanged.

\begin{algorithm}[!ht]
\caption{RWKV-7 PoST: Time-Mixing Forward Pass}
\label{alg:rwkv_post}
\begin{algorithmic}[1]
\Require Input $x \in \R^{B \times T \times C}$, PoST parameters $\theta_w \in \R$, $\delta_w \in \R^{C-1}$, LoRA weights, position offset $t_0 \ge 0$.
\Ensure Output $y \in \R^{B \times T \times C}$
\Statex
\State {\color{blue} /* Step 1: Spectral Reparameterization (Definition~\ref{def:post_map}) */}
\State $g_j \gets \softplus(\delta_{w,j})$ for $j = 1, \ldots, C-1$
\State $w_{0,k} \gets \theta_w + \sum_{j=1}^{k-1} g_j$ for $k = 1, \ldots, C$ \Comment{$w_{0,1} < \cdots < w_{0,C}$}
\Statex
\State {\color{blue} /* Step 2: Taper exponents (Corollary~\ref{cor:sigmoid_taper}) */}
\State $p_k \gets \log\sigma(w_{0,k})$; \quad $\bar{G} \gets (p_C - p_1)/(C-1)$
\State $\alpha_k \gets \clamp\bigl(\tfrac{C-k}{C-1} + \tfrac{(p_k - p_1)-(k-1)\bar{G}}{\log T_{\mathrm{train}}},0,1\bigr)$ for $k=1,\ldots,C$
\Statex
\State {\color{blue} /* Step 3: Additive logit-space position scaling (Eq.~\ref{eq:rwkv_post_additive}) */}
\State $x_w \gets x + (x_{\mathrm{prev}} - x) \circ \mu_w$ \Comment{Token-shift mixing}
\State $\ell_{l,k} \gets w_{0,k} + [\mathrm{LoRA}(x_{w,l})]_k - \alpha_k\cdot\log(t_0+l)$ for $l \in [T]$
\State $w_{l,k} \gets -\lambda\cdot\sigma(\ell_{l,k})$ \Comment{$w \in (-\lambda,0)^{B\times T\times C}$}
\Statex
\State {\color{blue} /* Step 4: Standard RWKV-7 WKV recurrence */}
\State Compute $r, k, v, a, \hat{\kappa}$ via standard RWKV-7 projections
\State $y \gets \mathrm{WKV7}(r, w, k, v, \hat{\kappa}, a)$
\State \Return $y$
\end{algorithmic}
\end{algorithm}


\subsection{Gated DeltaNet PoST}
\label{subsec:gdn_post}

We additionally instantiate PoST on Gated DeltaNet~\citep{ykh25}, demonstrating compatibility with architectures utilizing matrix-valued linear attention with data-dependent forget gates.

\paragraph{Decay mechanism.}
Gated DeltaNet uses an exponential forget gate with data-dependent modulation to update its matrix-valued hidden state. The log-decay mechanism operates directly in the log space, similar to Mamba-2.

\paragraph{Implementation.}
Spectral Reparameterization applies identically to Mamba-2 (Algorithm~\ref{alg:post_generic}). We replace the per-head learnable bias with strictly ordered rates generated by the cumulative-softplus map (Definition~\ref{def:post_map}). The position-adaptive scaling factor is applied directly inside the exponential gate, ensuring scale-free retention while preserving fine-grained data-dependent modulation.


\subsection{Other Architecture Instantiations}
\label{subsec:arch_specific}

Both GLA and RetNet~\citep{syc+23} use a fixed per-head scalar decay $\gamma_h \in (0,1)$.  Since both architectures share the same decay structure, applying PoST yields an identical reparameterization; PoST-GLA and PoST-RetNet reduce to the same model and are reported together in our experiments (Table~\ref{tab:mqar}).  Full pseudocode is in Appendix~\ref{app:arch_pseudocode}.

\section{Experiments}
\label{sec:experiments}

We evaluate the PoST framework through three complementary experimental settings: \textbf{Multi-Query Associative Recall (MQAR)}~\citep{aet+24}, a controlled synthetic benchmark that tests associative recall under length extrapolation; \textbf{zero-shot language modeling benchmarks}, which confirm that the spectral reparameterization consistently improves general language modeling capabilities; and \textbf{Needle-In-A-Haystack (NIAH)}, which tests both single-needle and multi-needle long-range verbatim retrieval.  We compare PoST-enhanced models against their standard baselines on Mamba-2~\citep{dg24}, RWKV-7, and Gated DeltaNet.


\subsection{Multi-Query Associative Recall}
\label{sec:mqar}

\paragraph{Setup.}
The MQAR task~\citep{aet+24} embeds $K$ key--value pairs in a sequence of total length $T$ and queries the model to retrieve all $K$ values.
We set $K = T/4$ and train $2$-layer models at $T_{\mathrm{train}} = 512$ using a four-stage curriculum that ramps $K$ from $16$ to $128$, then evaluate at $T \in \{512, 1024, 2048, 4096\}$ ($1\times$--$8\times$), so that both the number of stored associations and the distractor length grow simultaneously at test time.
We compare five architectures (Mamba-2~\citep{dg24}, RWKV-7~\citep{pzg+25}, Gated DeltaNet~\citep{ykh25}, Gated Linear Attention (GLA)~\citep{yws+24}, and RetNet~\citep{syc+23}) together with their PoST-enhanced counterparts, across model widths $d \in \{512, 256\}$, sweeping learning rates and reporting the best per model.
To ensure a fair comparison, all architectures use the same base number of heads at each $d_{\mathrm{model}}$; state-size equalization is achieved by adjusting $d_{\mathrm{state}}$ (Mamba-2); see Appendix~\ref{app:mqar_state}.
All training uses BF16 mixed precision.
All experiments use the Zoology framework~\citep{aet+24}.
Full experimental details, including curriculum schedule, sweep axes, and test configurations, are in Appendix~\ref{app:mqar}.

\paragraph{Results.}
Table~\ref{tab:mqar} summarizes the results; accuracy curves are in Appendix~\ref{app:mqar_results}.

\begin{table}[!ht]
\centering
\caption{\textbf{MQAR capacity accuracy (\%)} at each test length $T$ with $K{=}T/4$ key--value pairs.
All models trained at $T{=}512$; longer lengths are out-of-distribution.
GLAPoST and RetNetPoST converge to the same model under the PoST framework.
Across almost all settings, PoST outperforms its baseline; the sole exception is GLA at state$=64K$, where the baseline edges ahead ($71.5$ vs.\ $69.7$ avg); GLAPoST recovers the lead at state$=32K$ and $16K$.}
\label{tab:mqar}
\small
\resizebox{\linewidth}{!}{%
\begin{tabular}{l|cccc|c|cccc|c|cccc|c}
\toprule
 & \multicolumn{5}{c|}{state$=64K$} & \multicolumn{5}{c|}{state$=32K$} & \multicolumn{5}{c}{state$=16K$} \\
\cmidrule(lr){2-6} \cmidrule(lr){7-11} \cmidrule(lr){12-16}
\textbf{Model} & 512 & 1K & 2K & 4K & \textbf{Avg} & 512 & 1K & 2K & 4K & \textbf{Avg} & 512 & 1K & 2K & 4K & \textbf{Avg} \\
\midrule
Mamba-2              & \textbf{100.0} & 96.8 & 62.2 & 18.9 & 69.5 & 99.2 & 85.2 & 41.3 & 11.6 & 59.4 & 99.3 & 80.6 & 31.2 & 5.7 & 54.2 \\
$\quad$+PoST       & \textbf{100.0} & \textbf{97.4} & \textbf{68.3} & \textbf{25.1} & \textbf{72.7} & \textbf{99.8} & \textbf{92.1} & \textbf{51.6} & \textbf{13.2} & \textbf{64.2} & \textbf{99.6} & \textbf{87.8} & \textbf{44.1} & \textbf{12.7} & \textbf{61.0} \\
\midrule
RWKV-7               & \textbf{100.0} & \textbf{100.0} & 96.1 & 39.2 & 83.8 & \textbf{100.0} & \textbf{100.0} & 80.1 & 9.5 & 72.4 & \textbf{100.0} & 95.2 & 46.0 & 10.8 & 63.0 \\
$\quad$+PoST       & \textbf{100.0} & \textbf{100.0} & \textbf{98.5} & \textbf{52.9} & \textbf{87.8} & \textbf{100.0} & \textbf{100.0} & \textbf{98.0} & \textbf{28.5} & \textbf{81.6} & \textbf{100.0} & \textbf{99.3} & \textbf{70.9} & \textbf{18.8} & \textbf{72.2} \\
\midrule
Gated DeltaNet                  & \textbf{100.0} & \textbf{100.0} & 92.0 & 42.4 & 83.6 & \textbf{100.0} & 96.4 & 56.7 & 15.9 & 67.2 & 99.8 & 82.7 & 31.7 & 7.4 & 55.4 \\
$\quad$+PoST       & \textbf{100.0} & \textbf{100.0} & \textbf{95.3} & \textbf{48.4} & \textbf{85.9} & \textbf{100.0} & \textbf{99.9} & \textbf{88.9} & \textbf{39.6} & \textbf{82.1} & \textbf{99.9} & \textbf{86.5} & \textbf{35.7} & \textbf{8.7} & \textbf{57.7} \\
\midrule
GLA                  & \textbf{100.0} & \textbf{97.8} & \textbf{67.2} & \textbf{20.8} & \textbf{71.5} & \textbf{100.0} & \textbf{97.7} & 50.3 & 7.6 & 63.9 & 99.8 & 88.5 & 38.7 & 7.8 & 58.7 \\
$\quad$+PoST       & \textbf{100.0} & 96.0 & 62.1 & 20.7 & 69.7 & 99.9 & 93.9 & \textbf{54.8} & \textbf{16.9} & \textbf{66.4} & \textbf{99.9} & \textbf{93.1} & \textbf{50.7} & \textbf{12.2} & \textbf{64.0} \\
\midrule
RetNet               & 99.9 & 47.1 & 2.3 & 0.0 & 37.3 & \textbf{99.9} & 63.2 & 6.0 & 0.3 & 42.3 & 96.8 & 16.8 & 0.7 & 0.0 & 28.6 \\
$\quad$+PoST       & \textbf{100.0} & \textbf{96.0} & \textbf{62.1} & \textbf{20.7} & \textbf{69.7} & \textbf{99.9} & \textbf{93.9} & \textbf{54.8} & \textbf{16.9} & \textbf{66.4} & \textbf{99.9} & \textbf{93.1} & \textbf{50.7} & \textbf{12.2} & \textbf{64.0} \\
\bottomrule
\end{tabular}%
}
\end{table}


\subsection{Language Model Pretraining and Evaluation}
\label{sec:lm_pretrain}

\paragraph{Setup.}
We pretrain Mamba-2, RWKV-7, and Gated DeltaNet language models on FineWeb-Edu~\citep{lbds24} at context length $T_{\mathrm{train}} = 2{,}048$ at ${\sim}$180M parameters, with Mamba-2 additionally evaluated at ${\sim}$440M.
Within each scale, the models share identical hyperparameters and differ \emph{only} in decay parameterization: the baseline uses the default initialization, while PoST uses the Spectral Reparameterization (Definition~\ref{def:post_map}) with position-adaptive decay scaling (Definition~\ref{def:position_adaptive_ssm}).
Full architecture and training details are in Appendix~\ref{app:lm_eval}.

\paragraph{Zero-Shot Evaluation.}
We evaluate all models on seven standard zero-shot benchmarks using the Language Model Evaluation Harness~\citep{gta+24}.  Table~\ref{tab:lm_eval} reports the results.

\begin{table}[!ht]
\centering
\caption{\textbf{Downstream zero-shot evaluations.}
PoST achieves consistently better average performance than the baseline across all benchmarks at 180M and 440M scales, indicating that the spectral reparameterization consistently, though modestly, improves general language modeling capabilities.}
\label{tab:lm_eval}
\small
\resizebox{\linewidth}{!}{%
\begin{tabular}{@{}l cc ccccccc@{}}
\toprule
\textbf{Model}
  & \multicolumn{2}{c}{LAMBADA} & HellaSwag & PIQA
  & ARC-Easy & ARC-Challenge
  & WinoGrande & OpenBookQA & Avg \\
  & {\scriptsize acc$\uparrow$} & {\scriptsize ppl$\downarrow$} & {\scriptsize acc$_\text{n}\uparrow$} & {\scriptsize acc$\uparrow$}
  & {\scriptsize acc$\uparrow$} & {\scriptsize acc$_\text{n}\uparrow$}
  & {\scriptsize acc$\uparrow$} & {\scriptsize acc$_\text{n}\uparrow$} & \\
\midrule
Mamba-2 180M
  & \textbf{21.6} & \textbf{145.4} & 31.1 & 62.9
  & \textbf{50.4} & 24.7
  & 49.6 & \textbf{30.6} & 38.7 \\
$\quad$+PoST
  & 21.5 & 148.2 & \textbf{31.3} & \textbf{63.2}
  & 50.1 & \textbf{24.9}
  & \textbf{50.6} & 30.0 & \textbf{38.8} \\
\midrule
RWKV-7 180M
  & 27.9 & \textbf{69.6} & 32.1 & \textbf{63.1}
  & 49.7 & \textbf{25.7}
  & 51.3 & 29.0 & 39.8 \\
\quad+PoST
  & \textbf{28.3} & 71.9 & 32.1 & 62.9
  & \textbf{52.1} & 25.3
  & \textbf{51.8} & \textbf{32.0} & \textbf{40.6} \\
\midrule
Gated DeltaNet 180M
  & 23.8 & \textbf{94.5} & \textbf{31.9} & 62.7
  & 49.6 & 24.2
  & \textbf{51.1} & 30.6 & 39.1 \\
$\quad$+PoST
  & \textbf{25.2} & 95.6 & 31.5 & \textbf{62.9}
  & \textbf{51.5} & \textbf{25.3}
  & 50.7 & \textbf{31.8} & \textbf{39.8} \\
\midrule
Mamba-2 440M
  & 24.1 & 77.3 & \textbf{37.7} & 65.3
  & \textbf{57.7} & \textbf{27.2}
  & 50.4 & \textbf{32.8} & 42.2 \\
$\quad$+PoST
  & \textbf{28.0} & \textbf{62.6} & 37.5 & 65.3
  & 56.6 & 26.2
  & \textbf{51.4} & 32.6 & \textbf{42.5} \\
\bottomrule
\end{tabular}%
}
\end{table}

As detailed in Table~\ref{tab:lm_eval}, these results confirm that the PoST spectral reparameterization consistently, though modestly, improves average downstream performance alongside its gains in long-context retrieval.

\paragraph{Empirical Timescale Analysis.}
\label{sec:empirical_timescales}

\begin{figure}[!ht]
    \centering
    \includegraphics[width=\linewidth]{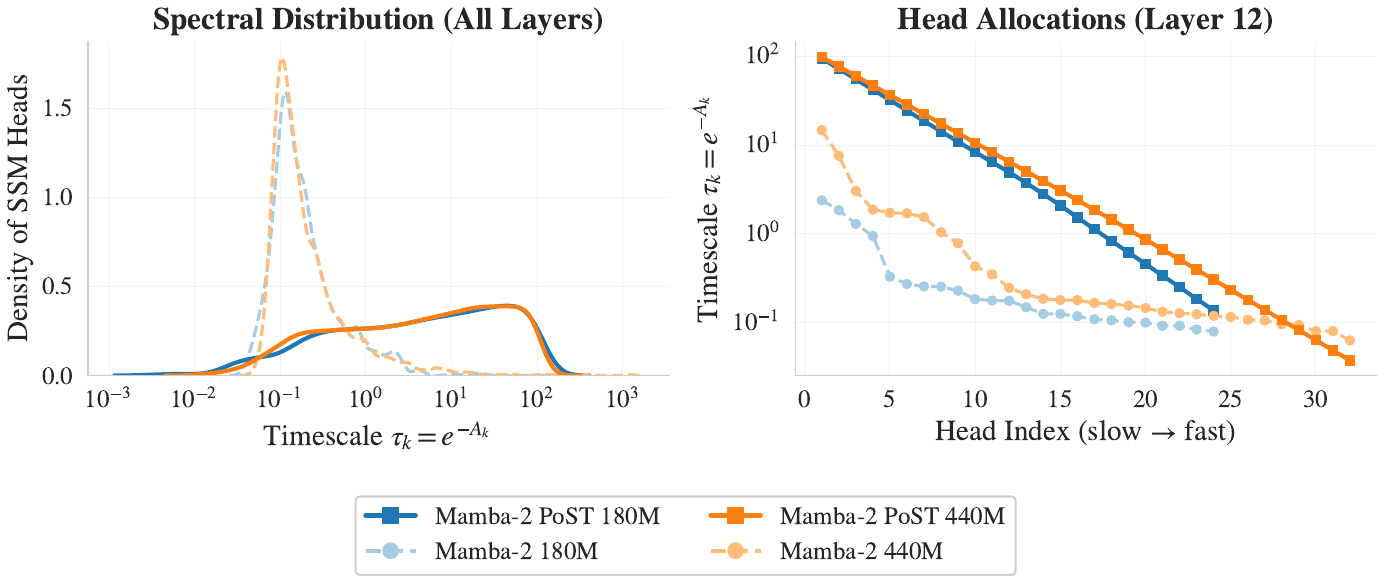}
    \caption{\textbf{Empirical timescale distribution $\tau_k = e^{-A_k}$ across trained models.} \textbf{(Left)} A kernel density estimate (KDE) over all depths shows Mamba-2 models suffering from a severe minimum gap collapse (density clumping into narrow spikes), while PoST strictly enforces a broad geometric long-tail distribution. \textbf{(Right)} Head timescale allocation within a single representative layer (Layer 12). Mamba-2 models flatten out (allocating many heads to identical timescales), wasting capacity. PoST forms a perfect straight line on the log-scale, empirically proving rigorous geometric spacing.}
    \label{fig:post_timescales}
\end{figure}

\begin{figure}[h!]
    \centering
    \includegraphics[width=\linewidth]{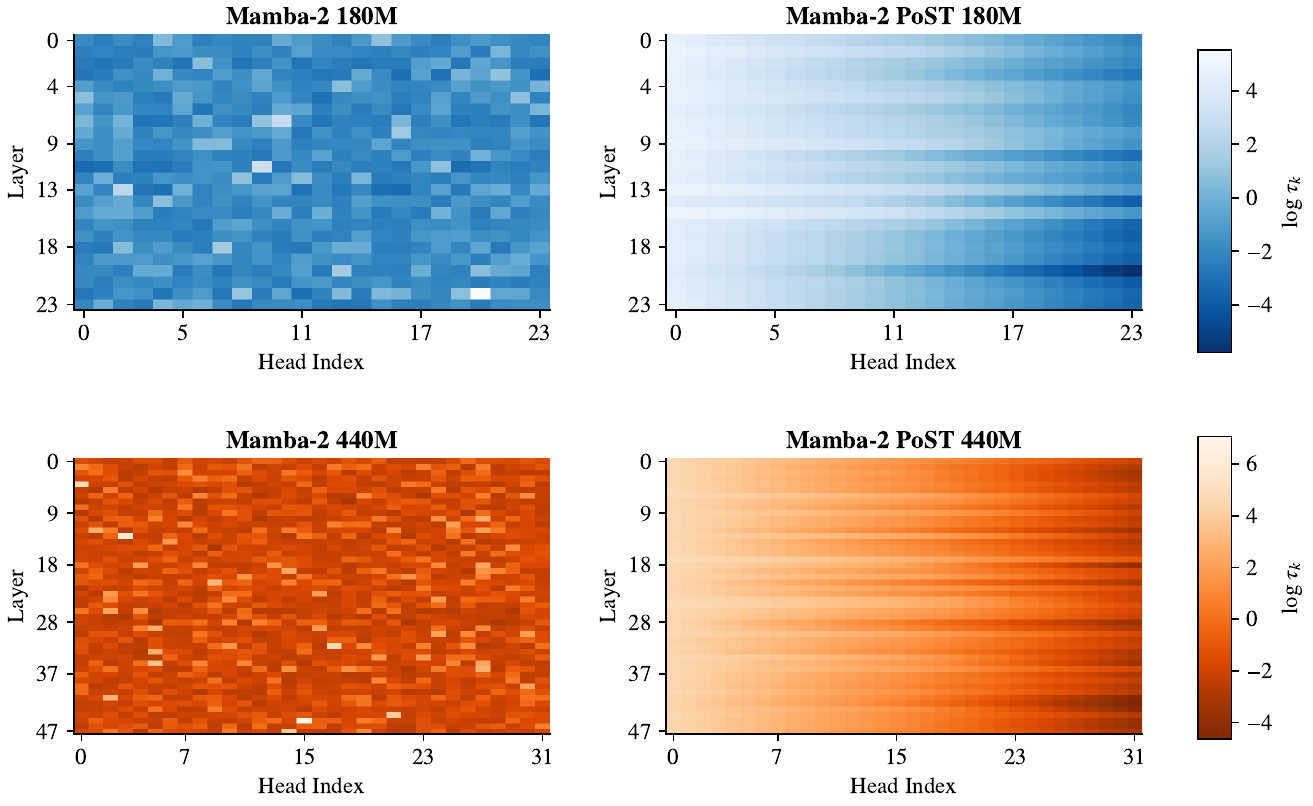}
    \caption{\textbf{Layer$\times$Head heatmap of learned log-timescales $\log\tau_k = -A_k$ across all layers.}
    Each cell encodes the log-timescale of a single SSM head at its \emph{actual} model index at a given layer (top to bottom); no manual reordering is applied.
    \textbf{Top row (180M):} The baseline Mamba-2 heatmap is nearly uniform in color throughout: all heads at every layer collapse to a narrow band of fast timescales, wasting state capacity.  The PoST counterpart displays a smooth left-to-right gradient because the Spectral Reparameterization \emph{structurally} enforces $p_1 < p_2 < \cdots < p_N$ via a cumulative-softplus construction: the ordering is an intrinsic property of the learned weights, not a visualization artifact.  This gradient is consistent across \emph{every} layer, confirming that geometric spectral order is a global, depth-invariant property.
    \textbf{Bottom row (440M):} The same contrast holds at larger scale (48 layers, 32 heads).  The PoST model preserves a wider dynamic range (larger spread between minimum and maximum $\log\tau_k$) than the baseline, directly reflecting the broader effective memory horizon predicted by Theorem~\ref{thm:optimality}.}
    \label{fig:post_heatmap}
\end{figure}

To verify that PoST structurally enforces the optimal geometric memory allocation derived in Section~\ref{sec:theory}, we analyze the learned timescales $\tau_k = e^{-A_k}$ of the 180M and 440M pretrained models. As visualized in Figure~\ref{fig:post_timescales} (Left), empirical inspections of pre-trained Mamba-2 models reveal this severe gap collapse: density plots show that the vast majority of heads collapse toward similar fast timescales, wasting state capacity and leaving critical low-frequency gaps. In Figure~\ref{fig:post_timescales} (Right), we confirm that Spectral Reparameterization strictly enforces a wide, geometrically spaced memory distributed across all available heads (forming a perfect linear progression on a log scale). This validates that PoST avoids the severe head redundancy seen in standard initializations and fully utilizes the model's state capacity.

Figure~\ref{fig:post_heatmap} extends this analysis to the full joint Layer$\times$Head structure, displaying \emph{raw} head indices without any sorting.  The baseline heatmap is scattered and nearly uniform across both axes, confirming that this minimum gap collapse is a pervasive, depth-invariant pathology: every layer independently collapses to similar fast timescales, leaving slow-timescale memory entirely unserviced.  PoST eliminates this pathology: the smooth color gradient across head-index and layer dimensions is \emph{not} a product of sorting; it emerges directly from the cumulative-softplus Spectral Reparameterization, which ties the ordering of learned weights to their head index by construction.  This provides direct visual confirmation of the layer-invariant non-degeneracy guaranteed by Proposition~\ref{thm:nondegeneracy}.

\begin{figure}[!ht]
    \centering
    \includegraphics[width=\linewidth]{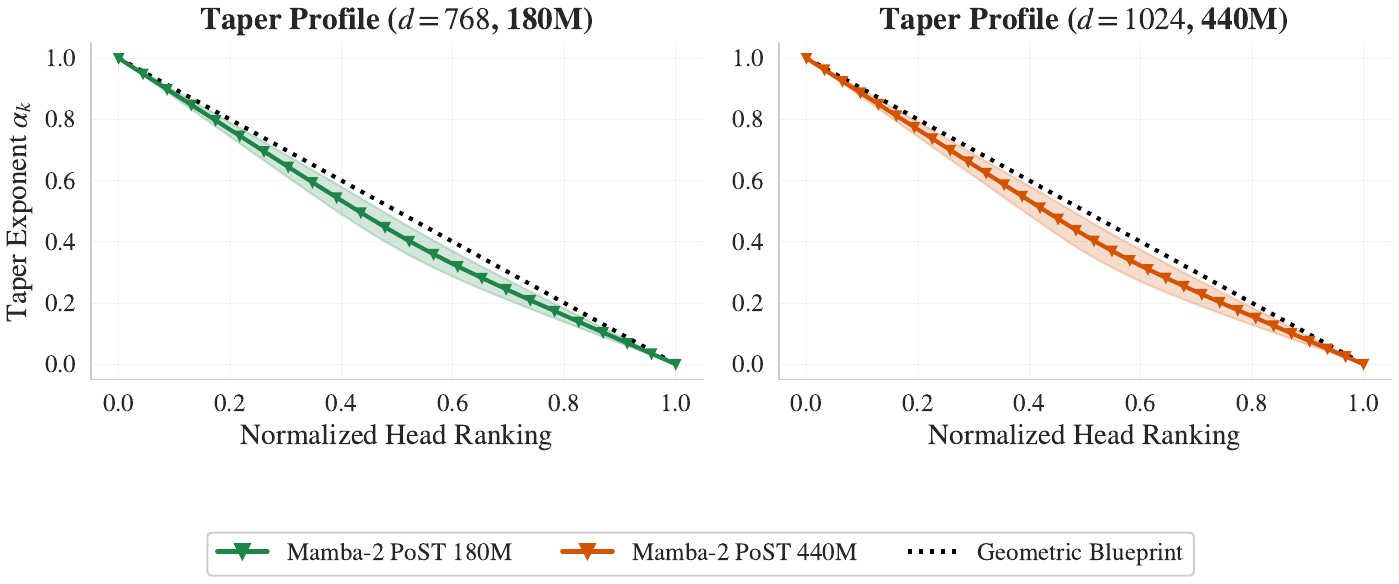}
    \caption{\textbf{Learned Normalization Taper $\alpha_k$}. We examine the empirical $\alpha_k$ distributions across all layers of the pre-trained Mamba-2 PoST 180M and 440M models. Rather than fixing $\alpha_k$ to the linear $(N-k)/(N-1)$ blueprint (dashed black line), PoST allows data-dependent parameter adjustments and adaptively recomputes the optimal $\alpha_k$ for each head to enforce rigid geometric spacing (Proposition~\ref{thm:taper_general}). Across both model scales, the empirical average taper strongly aligns closely with the strict linear ideal, revealing that natural language optimization inherently converges to the hierarchically distributed memory blueprint.}
    \label{fig:post_alphas_comparison}
\end{figure}

As shown in Figure~\ref{fig:post_alphas_comparison}, the position-adaptive parameterization functions precisely as the theoretical blueprint intends. While PoST allows the spectrum itself to remain learnable through optimization on the FineWeb-Edu dataset, the resulting adapted $\alpha_k$ values (computed via Equation~\ref{eq:general_taper}) follow the theoretical linear curve. This provides unambiguous empirical evidence that optimization on natural language gravitates toward uniform memory allocation across sequence hierarchies.

\paragraph{Long-Context Retrieval: NIAH.}
\label{sec:niah}
We evaluate the pretrained models on the \textbf{NIAH} (Needle-In-A-Haystack) benchmark, which embeds a target ``needle'' sentence within a long distractor context and asks the model to retrieve it verbatim.
We test both \emph{single-needle} variants (Single-1/2/3) and \emph{multi-needle} variants (MultiKey, MultiQuery, MultiValue) at $T \in \{1{,}024, 2{,}048, 4{,}096\}$.
Table~\ref{tab:retrieval} presents the results.

\begin{table}[!ht]
\centering
\caption{\textbf{NIAH long-context retrieval results.}
Single-needle and multi-needle variants evaluated at $T \in \{1\text{K}, 2\text{K}, 4\text{K}\}$.
PoST significantly improves retrieval for Mamba-2 at both 180M and 440M scales, particularly as context length grows. For Gated DeltaNet, PoST yields a moderate overall improvement. For RWKV-7, whose baseline already achieves strong retrieval, performance remains highly competitive.}
\label{tab:retrieval}
\small
\resizebox{\linewidth}{!}{%
\begin{tabular}{@{}l ccc ccc ccc ccc ccc ccc c@{}}
\toprule
& \multicolumn{9}{c}{\textbf{Single-Needle}} & \multicolumn{9}{c}{\textbf{Multi-Needle}} & \\
\cmidrule(lr){2-10} \cmidrule(lr){11-19}
& \multicolumn{3}{c}{Single-1} & \multicolumn{3}{c}{Single-2} & \multicolumn{3}{c}{Single-3} & \multicolumn{3}{c}{MultiKey} & \multicolumn{3}{c}{MultiQuery} & \multicolumn{3}{c}{MultiValue} & \\
\cmidrule(lr){2-4} \cmidrule(lr){5-7} \cmidrule(lr){8-10} \cmidrule(lr){11-13} \cmidrule(lr){14-16} \cmidrule(lr){17-19}
\textbf{Model} & 1K & 2K & 4K & 1K & 2K & 4K & 1K & 2K & 4K & 1K & 2K & 4K & 1K & 2K & 4K & 1K & 2K & 4K & Avg \\
\midrule
Mamba-2 180M & 44.0 & 5.6 & 0.2 & 15.4 & 3.6 & 0.8 & 0.0 & 0.0 & 0.0 & 9.0 & 6.0 & 1.6 & 8.5 & 5.5 & 1.4 & 4.9 & 4.5 & 1.7 & 6.3 \\
$\quad$+PoST & \textbf{95.6} & \textbf{47.4} & \textbf{2.0} & \textbf{71.0} & \textbf{13.2} & \textbf{8.4} & \textbf{4.8} & \textbf{0.6} & \textbf{1.8} & \textbf{17.0} & \textbf{16.0} & \textbf{6.4} & \textbf{12.6} & \textbf{12.6} & \textbf{3.1} & \textbf{11.1} & \textbf{12.0} & \textbf{3.5} & \textbf{18.8} \\
\midrule
RWKV-7 180M & 99.6 & \textbf{97.4} & \textbf{62.6} & 66.8 & \textbf{11.8} & \textbf{12.4} & 0.0 & 0.0 & 0.0 & 17.8 & \textbf{16.8} & \textbf{8.8} & 14.1 & \textbf{13.1} & \textbf{3.3} & \textbf{16.2} & \textbf{11.5} & \textbf{7.0} & \textbf{25.5} \\
$\quad$+PoST & \textbf{99.8} & 93.6 & 57.8 & \textbf{90.2} & 11.6 & 5.4 & \textbf{3.0} & \textbf{0.2} & 0.0 & 17.8 & 14.0 & 5.8 & \textbf{16.1} & 7.5 & 0.8 & 15.3 & 10.8 & 3.0 & 25.1 \\
\midrule
Gated DeltaNet 180M & \textbf{100.0} & 97.8 & \textbf{85.8} & 78.6 & 13.2 & 4.4 & 7.6 & 1.6 & 0.8 & 17.2 & 22.8 & 6.0 & \textbf{21.5} & \textbf{18.9} & \textbf{7.8} & 12.6 & 18.6 & 4.5 & 28.9 \\
$\quad$+PoST & 99.0 & \textbf{99.6} & 77.6 & \textbf{98.0} & \textbf{14.4} & \textbf{12.6} & \textbf{10.2} & \textbf{1.8} & \textbf{1.0} & \textbf{21.4} & \textbf{25.2} & \textbf{12.2} & 12.8 & 5.6 & 0.9 & \textbf{16.9} & \textbf{19.1} & \textbf{8.8} & \textbf{29.8} \\
\midrule
Mamba-2 440M & 98.2 & 63.8 & \textbf{30.4} & 94.8 & 24.2 & 2.6 & 31.4 & 13.6 & \textbf{4.0} & \textbf{16.2} & 13.4 & \textbf{3.4} & \textbf{14.3} & 12.0 & \textbf{3.2} & \textbf{8.6} & 4.0 & \textbf{1.2} & 24.4 \\
$\quad$+PoST & \textbf{99.8} & \textbf{77.6} & 16.2 & \textbf{98.8} & \textbf{34.2} & \textbf{7.2} & \textbf{60.4} & \textbf{30.4} & 1.4 & 15.2 & \textbf{19.6} & 2.2 & 9.4 & \textbf{15.4} & 1.1 & 3.5 & \textbf{10.5} & 0.2 & \textbf{28.0} \\
\bottomrule
\end{tabular}%
}
\end{table}

NIAH retrieval reveals a clear architecture-dependent pattern. As shown in Table~\ref{tab:retrieval}, PoST significantly improves single-needle and multi-needle retrieval for Mamba-2 at both 180M and 440M scales, with gains becoming more pronounced on harder variants (Single-3 and multi-needle tasks) and at longer contexts. For Gated DeltaNet, PoST yields a moderate overall improvement ($28.9 \to 29.8$ avg). For RWKV-7, whose baseline already achieves the strongest retrieval among the tested architectures ($25.5$ avg), PoST performs comparably ($25.1$ avg); the small difference falls within the variance expected from the spectral restructuring not providing additional benefit when the baseline already maintains a well-distributed decay spectrum via its sigmoid gate and power-law initialization. These results suggest that PoST provides the largest gains for architectures whose baseline parameterization is most susceptible to spectral collapse (Mamba-2), while preserving performance for architectures with more robust native spectral properties.

\paragraph{Remark.}
Within each model pair, the models are trained with identical hyperparameters, data, and compute; the only difference is the decay parameterization and the position-adaptive scaling.
The zero-shot results demonstrate that the spectral restructuring consistently improves performance on standard benchmarks, while the single- and multi-needle NIAH results show that the gains from PoST manifest significantly on memory-intensive and long-context tasks for Mamba-2, directly isolating the long-range memory advantage predicted by the theory.
MQAR (Section~\ref{sec:mqar}) provides further complementary evidence in a controlled synthetic setting.


\subsection{Discussion}
\label{sec:discussion}

The experimental settings provide complementary evidence for the benefits of spectrally structured decay parameterization.
MQAR isolates the role of spectral structure in a controlled environment where the number of stored associations and the distractor length are precisely varied, directly testing associative recall under length extrapolation.
The zero-shot LM benchmarks show that the PoST reparameterization achieves consistent, though modest, improvements over the baseline, confirming that the spectral restructuring enhances general language modeling capabilities.
The NIAH tasks provide the strongest evidence for Mamba-2: as highlighted in Table~\ref{tab:retrieval}, PoST significantly outperforms the standard Mamba-2 baseline on single- and multi-needle retrieval at both 180M and 440M scales, particularly as context length and target count grow. For Gated DeltaNet, gains are moderate, while RWKV-7 performs comparably to its already strong baseline. This architecture-dependent pattern suggests that PoST provides the largest benefit when baseline spectral structure is poorly conditioned, as is the case for Mamba-2's standard S4D-Real initialization.

\paragraph{Limitations and ongoing work.}
The current LM and NIAH evaluations use 180M and 440M-parameter models trained on $4$--$9$B tokens.
We are actively scaling PoST to \textbf{1.5B parameters} trained on \textbf{30B tokens} for evaluation at a scale where architectural differences are more pronounced.
\section{Conclusion}
\label{sec:conclusion}
We introduced \textbf{Position-Adaptive Spectral Tapering (PoST)}, a comprehensive framework for sequential memory that prevents minimum-gap collapse via Spectral Reparameterization and achieves minimax-optimal state utilization via Position-Adaptive Scaling.  The framework is grounded in an information-theoretic derivation of optimal timescale allocation under approximate logarithmic equipartition, with a formal robustness guarantee showing graceful degradation when the equipartition condition holds only approximately.  In practice, the entire framework reduces to a two-line change in any compatible recurrent layer's forward pass, preserving both complexity and expressiveness.  Experiments across five major architectures (Mamba-2, RWKV-7, Gated DeltaNet, GLA, and RetNet) on MQAR, alongside full zero-shot language modeling and NIAH retrieval evaluations at 180M and 440M scales, confirm that PoST consistently improves zero-shot language modeling across all tested architectures and yields significant long-range retrieval gains for architectures with poorly conditioned baseline spectra, particularly Mamba-2.  We view PoST as a broadly applicable ``spectral hygiene'' primitive for the growing family of linear recurrent sequence models. Our implementation is open-sourced at \url{https://github.com/SiLifen/PoST}.

\section*{Acknowledgments}
The author thanks his parents for generously funding the computational resources used in this work.

\ifdefined\isarxiv
\bibliographystyle{alpha}
\bibliography{ref}
\else
\bibliographystyle{alpha}
\bibliography{ref}

@article{bm05,
  title     = {On approximation of functions by exponential sums},
  author    = {Beylkin, Gregory and Monz{\'o}n, Lucas},
  journal   = {Applied and Computational Harmonic Analysis},
  volume    = {19},
  number    = {1},
  pages     = {17--48},
  year      = {2005},
  publisher = {Elsevier}
}

@inproceedings{vsp+17,
  title     = {Attention is all you need},
  author    = {Vaswani, Ashish and Shazeer, Noam and Parmar, Niki and Uszkoreit, Jakob and Jones, Llion and Gomez, Aidan N and Kaiser, {\L}ukasz and Polosukhin, Illia},
  booktitle = {Advances in Neural Information Processing Systems},
  volume    = {30},
  year      = {2017}
}

@inproceedings{ggr22,
  title     = {Efficiently Modeling Long Sequences with Structured State Spaces},
  author    = {Gu, Albert and Goel, Karan and R{\'e}, Christopher},
  booktitle = {International Conference on Learning Representations},
  year      = {2022}
}

@article{gd23,
  title   = {Mamba: Linear-time sequence modeling with selective state spaces},
  author  = {Gu, Albert and Dao, Tri},
  journal = {arXiv preprint arXiv:2312.00752},
  year    = {2023}
}

@article{dg24,
  title   = {Transformers are SSMs: Generalized Models and Efficient Algorithms Through Structured State Space Duality},
  author  = {Dao, Tri and Gu, Albert},
  journal = {arXiv preprint arXiv:2405.21060},
  year    = {2024}
}

@book{t19,
  title     = {Approximation theory and approximation practice, extended edition},
  author    = {Trefethen, Lloyd N},
  year      = {2019},
  publisher = {SIAM}
}

@article{p65,
  title     = {Spacings},
  author    = {Pyke, Ronald},
  journal   = {Journal of the Royal Statistical Society: Series B (Methodological)},
  volume    = {27},
  number    = {3},
  pages     = {395--436},
  year      = {1965},
  publisher = {Wiley Online Library}
}

@article{gr89,
  title   = {Equilibrium distributions and degree of rational approximation of analytic functions},
  author  = {Gonchar, Andrei A and Rakhmanov, Evguenii A},
  journal = {Sbornik: Mathematics},
  volume  = {62},
  number  = {2},
  pages   = {305--348},
  year    = {1989}
}

@inproceedings{aet+24,
  title     = {Zoology: Measuring and Improving Recall in Efficient Language Models},
  author    = {Arora, Simran and Eyuboglu, Sabri and Timalsina, Aman and Johnson, Isys and Poli, Michael and Zou, James and Rudra, Atri and R{\'e}, Christopher},
  booktitle = {International Conference on Learning Representations},
  year      = {2024}
}

@article{ep94,
  title     = {Entropy and Long-Range Correlations in Literary English},
  author    = {Ebeling, Werner and P{\"o}schel, Thorsten},
  journal   = {Europhysics Letters},
  volume    = {26},
  number    = {4},
  pages     = {241--246},
  year      = {1994},
  publisher = {IOP Publishing}
}

@article{lt17,
  title   = {Criticality in Formal Languages and Statistical Physics},
  author  = {Lin, Henry W and Tegmark, Max},
  journal = {Entropy},
  volume  = {19},
  number  = {7},
  pages   = {299},
  year    = {2017}
}

@article{vc78,
  title     = {``{$1/f$} noise'' in music: Music from {$1/f$} noise},
  author    = {Voss, Richard F and Clarke, John},
  journal   = {The Journal of the Acoustical Society of America},
  volume    = {63},
  number    = {1},
  pages     = {258--263},
  year      = {1978},
  publisher = {Acoustical Society of America}
}

@book{a66,
  title     = {Lectures on functional equations and their applications},
  author    = {Acz{\'e}l, J{\'a}nos},
  year      = {1966},
  publisher = {Academic Press}
}

@book{gs64,
  title     = {Generalized Functions, Volume~1: Properties and Operations},
  author    = {Gel'fand, Israel M. and Shilov, Georgi E.},
  year      = {1964},
  publisher = {Academic Press}
}

@article{w75,
  title     = {The renormalization group: Critical phenomena and the {Kondo} problem},
  author    = {Wilson, Kenneth G},
  journal   = {Reviews of Modern Physics},
  volume    = {47},
  number    = {4},
  pages     = {773--840},
  year      = {1975},
  publisher = {APS}
}

@article{k90,
  title     = {Scaling and universality in statistical physics},
  author    = {Kadanoff, Leo P},
  journal   = {Physica A: Statistical Mechanics and its Applications},
  volume    = {163},
  number    = {1},
  pages     = {1--14},
  year      = {1990},
  publisher = {Elsevier}
}

@article{kmc+20,
  title   = {Scaling laws for neural language models},
  author  = {Kaplan, Jared and McCandlish, Sam and Henighan, Tom and Brown, Tom B and Chess, Benjamin and Child, Rewon and Gray, Scott and Radford, Alec and Wu, Jeffrey and Amodei, Dario},
  journal = {arXiv preprint arXiv:2001.08361},
  year    = {2020}
}

@inproceedings{lh19,
  title     = {Decoupled weight decay regularization},
  author    = {Loshchilov, Ilya and Hutter, Frank},
  booktitle = {International Conference on Learning Representations},
  year      = {2019}
}

@article{ggr20,
  title   = {{HiPPO}: Recurrent Memory with Optimal Polynomial Projections},
  author  = {Gu, Albert and Dao, Tri and Ermon, Stefano and Rudra, Atri and R{\'e}, Christopher},
  journal = {Advances in Neural Information Processing Systems},
  volume  = {33},
  year    = {2020}
}

@article{ggr22b,
  title   = {On the Parameterization and Initialization of Diagonal State Space Models},
  author  = {Gu, Albert and Gupta, Ankit and Goel, Karan and R{\'e}, Christopher},
  journal = {Advances in Neural Information Processing Systems},
  volume  = {35},
  year    = {2022}
}

@article{sgb+23,
  title   = {Simplified State Space Layers for Sequence Modeling},
  author  = {Smith, Jimmy T.H. and Warrington, Andrew and Linderman, Scott W.},
  journal = {International Conference on Learning Representations},
  year    = {2023}
}

@article{ggb22,
  title   = {Diagonal State Spaces are as Effective as Structured State Spaces},
  author  = {Gupta, Ankit and Gu, Albert and Berant, Jonathan},
  journal = {Advances in Neural Information Processing Systems},
  volume  = {35},
  year    = {2022}
}

@article{pbm+23,
  title   = {{RWKV}: Reinventing {RNNs} for the Transformer Era},
  author  = {Peng, Bo and Alcaide, Eric and Anthony, Quentin and Albalak, Alon and Arcadinho, Samuel and Biderman, Stella and Cao, Huanqi and Cheng, Xin and Chung, Michael and others},
  journal = {Findings of the Association for Computational Linguistics: EMNLP},
  year    = {2023}
}

@article{syc+23,
  title   = {Retentive Network: A Successor to Transformer for Large Language Models},
  author  = {Sun, Yutao and Dong, Li and Huang, Shaohan and Ma, Shuming and Xia, Yuqing and Xue, Jilong and Wang, Jianyong and Wei, Furu},
  journal = {arXiv preprint arXiv:2307.08621},
  year    = {2023}
}

@inproceedings{psa+22,
  title     = {Train Short, Test Long: Attention with Linear Biases Enables Input Length Generalization},
  author    = {Press, Ofir and Smith, Noah A. and Lewis, Mike},
  booktitle = {International Conference on Learning Representations},
  year      = {2022}
}

@article{slk+24,
  title   = {{RoFormer}: Enhanced Transformer with Rotary Position Embedding},
  author  = {Su, Jianlin and Ahmed, Murtadha and Lu, Yu and Pan, Shengfeng and Bo, Wen and Liu, Yunfeng},
  journal = {Neurocomputing},
  volume  = {568},
  pages   = {127063},
  year    = {2024}
}

@article{bpa23,
  title   = {{NTK}-Aware Scaled {RoPE} Allows {LLaMA} Models to Have Extended (8k+) Context Size},
  author  = {bloc97},
  journal = {Reddit post, r/LocalLLaMA},
  year    = {2023},
  note    = {\url{https://www.reddit.com/r/LocalLLaMA/comments/14lz7j5/}}
}

@article{pda+24,
  title   = {{YaRN}: Efficient Context Window Extension of Large Language Models},
  author  = {Peng, Bowen and Quesnelle, Jeffrey and Fan, Honglu and Shippole, Enrico},
  journal = {International Conference on Learning Representations},
  year    = {2024}
}

@article{aks+25,
  title   = {{MambaExtend}: A Training-Free Approach to Improve Long Context Extension of {Mamba}},
  author  = {Azizi, Seyedarmin and Kundu, Souvik and Sadeghi, Mohammad Erfan and Pedram, Massoud},
  journal = {International Conference on Learning Representations},
  year    = {2025}
}

@article{yxf+25,
  title   = {{LongMamba}: Enhancing {Mamba}'s Long-Context Capabilities via Training-Free Receptive Field Enlargement},
  author  = {Ye, Zhifan and Xia, Kejing and Fu, Yonggan and Dong, Xin and Hong, Jihoon and Yuan, Xiangchi and Diao, Shizhe and Kautz, Jan and Molchanov, Pavlo and Lin, Yingyan Celine},
  journal = {arXiv preprint arXiv:2504.16053},
  year    = {2025}
}

@article{bza+25,
  title   = {{DeciMamba}: Exploring the Length Extrapolation Potential of {Mamba}},
  author  = {Ben-Kish, Assaf and Zimerman, Itamar and Abu-Hussein, Shady and Cohen, Nadav and Globerson, Amir and Wolf, Lior and Giryes, Raja},
  journal = {International Conference on Learning Representations},
  year    = {2025}
}

@article{sba+25,
  title   = {Uncovering the Spectral Bias in Diagonal State Space Models},
  author  = {Solozabal, Ruben and Bojkovic, Velibor and AlQuabeh, Hilal and Inui, Kentaro and Tak{\'a}{\v{c}}, Martin},
  journal = {arXiv preprint arXiv:2508.20441},
  year    = {2025}
}

@article{yws+24,
  title   = {Gated Linear Attention Transformers with Hardware-Efficient Training},
  author  = {Yang, Songlin and Wang, Bailin and Shen, Yikang and Panda, Rameswar and Kim, Yoon},
  journal = {International Conference on Machine Learning},
  year    = {2024}
}

@article{dmr+24,
  title   = {Griffin: Mixing Gated Linear Recurrences with Local Attention for Efficient Language Models},
  author  = {De, Soham and Smith, Samuel L. and Fernando, Anushan and Botev, Aleksandar and Cristian-Muraru, George and Gu, Albert and Haroun, Ruba and Berrada, Leonard and Chen, Yutian and Srinivasan, Srivatsan and Desjardins, Guillaume and Doucet, Arnaud and Budden, David and Teh, Yee Whye and Pascanu, Razvan and De Freitas, Nando and Gulcehre, Caglar},
  journal = {arXiv preprint arXiv:2402.19427},
  year    = {2024}
}

@inproceedings{llc+26,
  title     = {Mamba-3: Improved Sequence Modeling Using State Space Principles},
  author    = {Lahoti, Aakash and Li, Kevin and Chen, Berlin and Wang, Caitlin and Bick, Aviv and Kolter, J. Zico and Dao, Tri and Gu, Albert},
  booktitle = {International Conference on Learning Representations},
  year      = {2026},
  note      = {OpenReview: \url{https://openreview.net/forum?id=HwCvaJOiCj}}
}

@article{pga+24,
  title   = {Eagle and Finch: {RWKV} with Matrix-Valued States and Dynamic Recurrence},
  author  = {Peng, Bo and Goldstein, Daniel and Anthony, Quentin and Albalak, Alon and Alcaide, Eric and Biderman, Stella and Cheah, Eugene and Du, Xingjian and Ferdinan, Teddy and others},
  journal = {arXiv preprint arXiv:2404.05892},
  year    = {2024}
}

@article{pzg+25,
  title   = {{RWKV-7} ``Goose'' with Expressive Dynamic State Evolution},
  author  = {Peng, Bo and Zhang, Ruichong and Goldstein, Daniel and Alcaide, Eric and Du, Xingjian and Hou, Haowen and Lin, Jiaju and Liu, Jiaxing and Lu, Janna and Merrill, William and Song, Guangyu and Tan, Kaifeng and Utpala, Saiteja and Wilce, Nathan and Wind, Johan S. and Wu, Tianyi and Wuttke, Daniel and Zhou-Zheng, Christian},
  journal = {arXiv preprint arXiv:2503.14456},
  year    = {2025}
}

@misc{gta+24,
  title        = {A Framework for Few-Shot Language Model Evaluation},
  author       = {Gao, Leo and Tow, Jonathan and Abbasi, Baber and Biderman, Stella and Black, Sid and DiPofi, Anthony and Foster, Charles and Golding, Laurence and Hsu, Jeffrey and Le Noac'h, Alain and others},
  year         = {2024},
  publisher    = {Zenodo},
  howpublished = {\url{https://github.com/EleutherAI/lm-evaluation-harness}},
  doi          = {10.5281/zenodo.10256836}
}

@inproceedings{ykh25,
  author    = {Songlin Yang and Jan Kautz and Ali Hatamizadeh},
  title     = {Gated Delta Networks: Improving Mamba2 with Delta Rule},
  booktitle = {International Conference on Learning Representations (ICLR)},
  year      = {2025},
  url       = {https://arxiv.org/abs/2412.06464}
}

@article{lbds24,
  title   = {{FineWeb-Edu}: The Finest Collection of Educational Content the {Web} has to Offer},
  author  = {Lozhkov, Anton and Ben Allal, Loubna and von Werra, Leandro and Wolf, Thomas},
  journal = {Hugging Face Blog},
  year    = {2024},
  note    = {\url{https://huggingface.co/spaces/HuggingFaceFW/blogpost-fineweb-v1}}
}

@article{dma+24,
  title   = {The {Llama} 3 Herd of Models},
  author  = {Dubey, Abhimanyu and Jauhri, Abhinav and Pandey, Abhinav and Kadian, Abhishek and Al-Dahle, Ahmad and Letman, Aiesha and Mathur, Akhil and Schelten, Alan and Yang, Amy and Fan, Angela and others},
  journal = {arXiv preprint arXiv:2407.21783},
  year    = {2024}
}
\fi


\newpage
\onecolumn
\appendix

\begin{center}
    \textbf{\LARGE Appendix }
\end{center}


\paragraph{Roadmap.}
The appendix is organized as follows.
Appendix~\ref{app:related} surveys related work on length extrapolation, spectral parameterizations, and linear recurrent architectures.
Appendix~\ref{app:theory_details} contains the full proofs and auxiliary lemmas for the results in Section~\ref{sec:theory} and Section~\ref{sec:methodology}.
Appendix~\ref{app:mqar} provides the complete MQAR experimental setup, including curriculum schedule, hyperparameter sweeps, and state-size equalization.
Appendix~\ref{app:lm_eval} gives additional language model pretraining and evaluation details.
Appendix~\ref{app:arch_pseudocode} presents architecture-specific pseudocode for applying PoST to RetNet and GLA.

\section{Related Work}
\label{app:related}

\paragraph{State space models and initialization.}
S4~\citep{ggr22} introduced structured state space models for long-range sequence modeling, using the HiPPO framework~\citep{ggr20} to initialize the transition matrix~$A$ from orthogonal polynomial projections.  The diagonal simplification S4D~\citep{ggr22b} showed that restricting to real-valued diagonal $A$ with linearly spaced eigenvalues ($\lambda_k = -(k+1)$) preserves most of the performance.  Subsequent models (S5~\citep{sgb+23}, DSS~\citep{ggb22}) continued to use independently parameterized, linearly spaced eigenvalues.  More recently, S4D-DFouT~\citep{sba+25} studies spectral bias in diagonal SSMs and proposes placing poles in the discrete Fourier domain for more uniform frequency coverage.  A common limitation of all these approaches is that spectral structure is imposed only at initialization and may be lost during training; moreover, they target the \emph{frequency response} of individual modes rather than the \emph{timescale allocation} that governs memory horizons.  PoST differs in two respects: Spectral Reparameterization enforces geometric spectral ordering \emph{throughout training} via a cumulative-softplus parameterization, and Position-Adaptive Scaling dynamically adjusts the spectrum to the observed context length.

\paragraph{Selective and input-dependent SSMs.}
Mamba~\citep{gd23} made the SSM parameters ($B$, $C$, $\Delta$) input-dependent, enabling content-aware gating.  Mamba-2~\citep{dg24} connected SSMs to structured attention via Structured State Space Duality.  Mamba-3~\citep{llc+26} further extends this line with improved discretization and state dynamics.  RWKV~\citep{pbm+23} and RetNet~\citep{syc+23} take complementary approaches to linear-time sequence modeling with element-wise decay.  These works focus on the architecture and computation of recurrent models; PoST is orthogonal, addressing the spectral structure of the decay spectrum within any diagonal linear recurrence.

\paragraph{Context extension for Mamba models.}
Several recent works address Mamba's degradation on sequences longer than those seen during training.  MambaExtend~\citep{aks+25} learns a single position-independent scaling factor per layer that uniformly rescales $\Delta$.  LongMamba~\citep{yxf+25} categorizes hidden channels into local and global based on receptive field length, then filters unimportant tokens from global channels to mitigate memory decay.  DeciMamba~\citep{bza+25} introduces a context-extension method built on a hidden filtering mechanism within the S6 layer, compressing the effective input to fit the model's trained receptive field.  All three methods are post-hoc, training-free interventions applied to frozen models.  PoST differs in three respects: it is active \emph{throughout training} (so learned weights co-adapt with the spectral structure), it provides \emph{position-adaptive} per-channel scaling via a closed-form formula (Proposition~\ref{thm:taper_general}), and it derives the target spectrum from first principles (Theorem~\ref{thm:optimality}).  Furthermore, PoST applies to any diagonal linear recurrence, not only Mamba.

\paragraph{Length extrapolation in Transformers.}
ALiBi~\citep{psa+22}, RoPE~\citep{slk+24} with NTK-Aware scaling~\citep{bpa23}, and YaRN~\citep{pda+24} modify positional encodings to extend the context window of Transformers.  The analogy to PoST is instructive: just as RoPE-based methods scale the frequency basis of positional encodings, PoST scales the timescale basis of the SSM decay spectrum.  However, PoST is grounded in approximation theory rather than positional encoding heuristics.

\paragraph{Power-law correlations and approximation theory.}
The theoretical foundation of PoST rests on the observation that natural language exhibits long-range correlations with approximate power-law decay~\citep{ep94, lt17}, echoing the broader $1/f$ noise literature~\citep{vc78}.  Our Condition~\ref{cond:hierarchical} formalizes this self-similar structure.  The connection between geometric pole placement and minimax-optimal approximation of power-law functions draws on classical results in rational approximation theory~\citep{gr89, bm05, t19}.  Beylkin and Monz\'{o}n~\citep{bm05} showed that exponential sums with geometrically spaced exponents achieve near-optimal approximation of smooth functions, a result we leverage in Theorem~\ref{thm:optimality}.  To our knowledge, PoST is the first work to connect these approximation-theoretic results to the spectral management of state space models.

\section{Theory Details}
\label{app:theory_details}

This appendix collects detailed proofs and analysis supplementing the theoretical results in Sections~\ref{sec:theory} and~\ref{sec:methodology}.

\subsection{State Energy Analysis}
\label{app:energy_analysis}

\begin{theorem}[Energy Scaling under the Linear Taper]
\label{thm:energy_scaling}
Mode~$k$ driven by unit-variance white noise up to position $t$ has expected energy (in the continuous-time approximation, valid up to $O(1/\tau_k)$ relative error for $\tau_k \gg 1$)
\begin{align*}
    E_k(t) = \frac{t^{\alpha_k}}{2 \ell_{t,k}} \bigl( 1 - \exp(-2 \ell_{t,k} \cdot t^{1-\alpha_k}) \bigr).
\end{align*}
The energy ratio between positions~$t_2$ and~$t_1$ (for $t_2 > t_1 > 0$) satisfies:
\begin{itemize}
    \item {\bf Part 1.} $\alpha_k = 0$: $E_k(t_2)/E_k(t_1) \to 1$ (position-invariant).
    \item {\bf Part 2.} $\alpha_k = 1$: $E_k(t_2)/E_k(t_1) = t_2/t_1$ (linear growth).
    \item {\bf Part 3.} General: scales asymptotically as $(t_2/t_1)^{\alpha_k}$ for $t_1, t_2 \gg 1$, and is strictly bounded between $(t_2/t_1)^{\alpha_k}$ and $(t_2/t_1)^1$.
\end{itemize}
\end{theorem}

\begin{proof}
In continuous time, a mode with timescale $\tau$ driven by unit-variance white noise accumulates expected energy $E = \frac{\tau}{2}(1 - e^{-2t/\tau})$.  In discrete time the exact variance is $(1 - e^{-2t/\tau})/(1 - e^{-2/\tau})$; since $(1 - e^{-2/\tau})^{-1} = \tau/2 + 1/2 + O(1/\tau)$, the continuous-time formula incurs a relative error of $O(1/\tau)$, negligible for long-lived modes ($\tau \gg 1$).  Using this approximation and substituting $\tau = \tau_k(t) = t^{\alpha_k}/\ell_{t,k}$:

Part~1: $\tau_k$ is constant, $E_k \to \tau_k/2$.
Part~2: $\tau_k(t) = t/\ell_{t,k}$, so $E_k(t) = \frac{t}{2\ell_{t,k}}(1 - e^{-2\ell_{t,k}})$, linear in $t$.
Part~3: For $\alpha_k \in (0, 1)$, the function $x \mapsto (1 - e^{-cx})/x$ is strictly decreasing, so $E_k(t)/t$ is strictly decreasing; hence the ratio $E_k(t_2)/E_k(t_1)$ is strictly bounded above by $(t_2/t_1)^1$. Conversely, $E_k(t)/t^{\alpha_k} \propto 1 - \exp(-2\ell_{t,k} t^{1-\alpha_k})$ is strictly increasing, so the ratio is strictly bounded below by $(t_2/t_1)^{\alpha_k}$. As $t \to \infty$, the exponential term decays to $0$, so the ratio converges asymptotically to $(t_2/t_1)^{\alpha_k}$.
\end{proof}

\begin{proposition}[Normalization Compatibility]
\label{prop:normalization}
Under the linear taper, the inter-mode relative energy asymptotically satisfies
$E_i(t_2)/E_j(t_2) \approx (t_2/t_1)^{\alpha_i - \alpha_j} \cdot E_i(t_1)/E_j(t_1)$ for large~$t_1, t_2$.
The maximum distortion (between extreme modes $i=1$, $j=N$) is governed by the factor $t_2/t_1$, meaning the deviation from unity approaches $|t_2/t_1 - 1|$, which is within the dynamic range that RMSNorm and the gating mechanism $y_{\mathrm{ssm}} \cdot \sigma(z)$ in Mamba-2~\citep{dg24} are designed to absorb for moderate extrapolation ratios.
\end{proposition}

\begin{proof}
By Theorem~\ref{thm:energy_scaling}, the energy of mode~$k$ at position~$t$ scales asymptotically as $E_k(t) \sim C_k t^{\alpha_k}$ for large~$t$.  Hence
$E_i(t_2)/E_j(t_2) \approx (t_2^{\alpha_i}/t_2^{\alpha_j}) \cdot (E_i(t_1)/E_j(t_1)) \cdot (t_1^{\alpha_j}/t_1^{\alpha_i}) = (t_2/t_1)^{\alpha_i - \alpha_j} \cdot E_i(t_1)/E_j(t_1)$.
For the extreme pair $i = 1$ ($\alpha_1 = 1$) and $j = N$ ($\alpha_N = 0$), the asymptotic distortion factor is $(t_2/t_1)^1 = t_2/t_1$.  Its deviation from unity is $|t_2/t_1 - 1|$.
\end{proof}

\subsection{Robustness under Approximate Equipartition}
\label{app:proof_approx_equipartition}

\begin{corollary}[Robustness under Approximate Equipartition, Formal Version of Corollary~\ref{cor:approx_equipartition}]
\label{cor:approx_equipartition:formal}
Provided the sequence distribution maintains bounded complexity $\epsilon \in [0, 1)$ according to Condition~\ref{cond:equipartition}, the optimal learned timescale exponents naturally adapt tightly around the geometric linear taper:
\begin{align*}
    \left|\alpha_k^* - \frac{N-k}{N-1}\right| \le \frac{2\epsilon}{1-\epsilon} \cdot \frac{N-k}{N-1}, \qquad k = 1, \ldots, N.
\end{align*}
\end{corollary}

\begin{proof}
Under approximate equipartition, each octave carries information $J(2^{j-1}) \in [J_0(1-\epsilon), J_0(1+\epsilon)]$.  The optimal allocation assigns channel density proportional to information density: an octave with higher $J$ ``deserves'' more channels.  Define the \emph{information CDF} on $[0, M]$ (where $M = \lfloor \log_2 t \rfloor$):
\begin{align*}
    F(u) := \frac{\int_0^u J(2^s)\,\mathrm{d}s}{\int_0^M J(2^s)\,\mathrm{d}s}.
\end{align*}
Since $J(2^s) \in [J_0(1-\epsilon), J_0(1+\epsilon)]$:
\begin{align*}
    \frac{(1-\epsilon)\,u}{(1+\epsilon)\,M} \;\le\; F(u) \;\le\; \frac{(1+\epsilon)\,u}{(1-\epsilon)\,M}.
\end{align*}
The optimal allocation places channel~$k$ at the log-timescale $u_k$ satisfying $F(u_k) = (N-k)/(N-1)$.  Under exact equipartition, $F_0(u) = u/M$ and $u_{k,0} = (N-k)/(N-1) \cdot M$, giving $\alpha_{k,0} = (N-k)/(N-1)$.

The CDF bounds yield $u_k \in \bigl[\tfrac{1-\epsilon}{1+\epsilon}\, u_{k,0},\; \tfrac{1+\epsilon}{1-\epsilon}\, u_{k,0}\bigr]$, so $\alpha_k^* = u_k / M$ satisfies
\begin{align*}
    \left|\alpha_k^* - \frac{N-k}{N-1}\right| \le \left(\frac{1+\epsilon}{1-\epsilon} - 1\right) \frac{N-k}{N-1} = \frac{2\epsilon}{1-\epsilon} \cdot \frac{N-k}{N-1}.
\end{align*}
The boundary values $\alpha_1^* = 1$ and $\alpha_N^* = 0$ are fixed by the problem constraints ($\tau_1 = t$, $\tau_N = 1$), independent of~$\epsilon$.
\end{proof}
\subsection{Approximation Penalty of Random Initialization}
\label{app:proof_random_approx}

\begin{lemma}[Approximation Penalty of Random Spacing, Formal Version of Lemma~\ref{thm:random_approx_limit}]
\label{thm:random_approx_limit:formal}
Under the conditions of Lemma~\ref{thm:collapse}, the maximum spectral gap $\Delta_{\max}^{(N)} := \max_{1 \le k < N} (p_{(k+1)} - p_{(k)})$ expands asymptotically as $\Omega(\frac{\log N}{N})$. Following Newman's bounds on rational approximation, the minimax error $E_N^{\mathrm{rand}}$ over $[1, T]$ for $K(s) = s^{-\beta}$ is structurally bottlenecked by this maximal spectral gap:
\begin{align*}
    E_N^{\mathrm{rand}} \ge C_1 \exp\left( - C_2 \frac{N}{\log N} \right),
\end{align*}
yielding a sub-exponential convergence rate that is strictly inferior to the optimal geometric rate $O(\exp(-c N/\log T))$. 
\end{lemma}

\begin{proof}
By the proof of Lemma~\ref{thm:collapse}, let $S_1, \dots, S_{N-1}$ denote the internal spacings of $N$ points drawn from a bounded density $f_P$. It is a classical result in extreme value theory~\citep{p65} that the maximum spacing $S_{\max}^{(N)}$ satisfies $\E[S_{\max}^{(N)}] = \Omega(\frac{\log N}{N})$. Since $\Delta_{\max}^{(N)}$ is bounded below proportionally by $S_{\max}^{(N)}$, the maximal spectral gap expands asymptotically as $\Omega(\frac{\log N}{N})$.

To connect this structural gap to the approximation error of the exponential sum $E_N^{\mathrm{rand}}$, we invoke Newman's bounds on the rational approximation of $x^{\alpha}$. The error of approximating $K(t) = t^{-\beta}$ via exponential sums is governed by the logarithmic capacity of the condenser defined by the nodes $p_k$. Whenever the maximum logarithmic gap $\Delta_{\max}$ strictly exceeds the expected uniform rate $\Theta(1/N)$, the capacity is strictly bottlenecked by this empty spectral region. The minimax error is bounded from below by:
\begin{align*}
    E_N^{\mathrm{rand}} \ge C_1 \exp\left( - \frac{C_2}{\Delta_{\max}^{(N)}} \right).
\end{align*}
Substituting $\Delta_{\max}^{(N)} = \Omega(\frac{\log N}{N})$, we obtain the sub-exponential lower bound:
\begin{align*}
    E_N^{\mathrm{rand}} \ge C_1 \exp\left( - C_2' \frac{N}{\log N} \right).
\end{align*}
This strictly forfeits the optimal exponential convergence rate $O(\exp(-c N/\log T))$ which is only achievable when all internal gaps are uniformly bounded by $O(1/N)$, as realized by the geometric progression of PoST.
\end{proof}


\subsection{Minimax Rates for Power-Law Approximation}
\label{app:proof_optimality}

We provide the formal statements and complete proofs for the approximation limits of linear and geometric spacing (corresponding to Lemma~\ref{thm:linear_limit} and Theorem~\ref{thm:optimality}). Assume throughout that $K(t) = t^{-\beta}$ with $\beta \in (0,1)$ and the approximation domain is $[1, T]$ with $T > 1$. Let $\Sigma_N$ denote the class of exponential sums with $N$ terms. Define the minimax error $E_N(K) := \inf_{g \in \Sigma_N} \|K - g\|_{L_\infty[1, T]}$.

\begin{lemma}[Linear Spacing Approximation Limit, Formal Version of Lemma~\ref{thm:linear_limit}]
\label{thm:linear_limit:formal}
If the log-decay rates are constrained to a linear grid $p_k = c \cdot k$, then the approximation error satisfies:
\begin{align*}
    E_N(K) \ge C_3 \exp\left( - \frac{C_4 N}{\sqrt{T}} \right),
\end{align*}
where $C_3, C_4 > 0$ depend on $\beta$. For modeling regimes where $N \ll \sqrt{T}$, the exponential convergence factor is heavily neutralized, leaving a practically algebraic rate bounded by $\Omega(N^{-\beta})$.
\end{lemma}

\begin{proof}
The proof proceeds in two steps: (1) reduce the linearly-spaced exponential sum to polynomial approximation; (2) apply a classical lower bound for polynomial approximation of singular functions.

\emph{Step 1: Reduction to polynomial approximation.}
When the decay rates are constrained to a linear grid $p_k = c \cdot k$ for $k = 1, \ldots, N$ and some $c > 0$, the exponential sum becomes
\begin{align*}
    g(t) = \sum_{k=1}^N w_k  e^{-ckt} = \sum_{k=1}^N w_k  z^k = P_N(z), \qquad z := e^{-ct},
\end{align*}
where $P_N$ is a polynomial of degree $N$ in $z$ with no constant term.  On the interval $t \in [1, T]$, we have $z \in [e^{-cT}, e^{-c}] =: [a_c, b_c] \subset (0,1)$.

To cover all relevant timescales of the kernel $K(t) = t^{-\beta}$ on $[1,T]$, the spacing $c$ must satisfy $cN \gtrsim 1$ (to resolve order-1 timescales) and $c \lesssim 1$ (otherwise the slowest mode $e^{-ct}$ decays too fast for $t \ge 1$).

\emph{Step 2: Lower bound via singularity analysis.}
In the $z$-variable, the target function is
\begin{align*}
    f(z) := K\left(\frac{-\log z}{c}\right) = c^{\beta} (-\log z)^{-\beta}.
\end{align*}
Consider the behavior as $z \to 1^-$: since $-\log z = (1-z) + O((1-z)^2)$, we have
\begin{align*}
    f(z) \sim c^{\beta}(1-z)^{-\beta}, \qquad z \to 1^-.
\end{align*}
Thus $f$ has an algebraic singularity of order $\beta$ at $z = 1$.  The interval $[a_c, b_c]$ lies inside $(0,1)$, but its right endpoint $b_c = e^{-c}$ satisfies $1 - b_c = 1 - e^{-c} = c + O(c^2)$.  Therefore, $b_c$ approaches the singularity as $c \to 0$.

By the classical Jackson--Bernstein converse theorems for polynomial approximation~\citep[Theorem~7.2]{t19}, if $f$ has an algebraic singularity of order $\beta$ at a point within distance $\delta$ of the approximation interval, then the best polynomial approximation of degree $N$ on that interval satisfies
\begin{align*}
    \inf_{P \in \mathcal{P}_N} \|f - P\|_{L_\infty[a_c, b_c]} \ge \frac{C_3'}{N^{\beta}},
\end{align*}
where $C_3'$ depends on $\beta$, $c$, and $T$.

Optimizing over $c > 0$ does not improve the rate.  To see this, note that $c$ controls a trade-off: decreasing $c$ brings $b_c$ closer to the singularity at $z = 1$ (making polynomial approximation harder), while increasing $c$ compresses the $z$-interval and reduces the polynomial's ability to represent the multi-scale structure of $K$.  In either regime, the algebraic singularity at $z = 1$ dominates the approximation rate.

More precisely, for any fixed $c > 0$, define $r_c := 1/(1 - b_c) = 1/(1 - e^{-c})$.  The Bernstein ellipse for the interval $[a_c, b_c]$ has semi-axis ratio determined by $r_c$, and the convergence rate of polynomial approximation is $O(\rho_c^{-N})$ where $\rho_c$ is the parameter of the largest Bernstein ellipse to which $f$ extends analytically.  Since $f$ has a singularity at $z = 1$, a distance $1 - b_c = O(c)$ from the interval endpoint, the Bernstein parameter satisfies
\begin{align*}
    \rho_c = 1 + \Theta\left(\sqrt{\frac{1-b_c}{b_c - a_c}}\right) = 1 + \Theta\left(\sqrt{\frac{c}{1 - e^{-cT}}}\right).
\end{align*}
For the optimal global coverage choice $c \asymp 1/T$ (which ensures the slowest mode spans the sequence length), we get $\rho_c = 1 + \Theta(1/\sqrt{T})$. The geometric convergence factor is thus restricted to $\rho_c^{-N} = \exp(-\Theta(N/\sqrt{T}))$. When combined with the algebraic singularity effect at $z=1$, classical weighted polynomial approximation theory yields the lower bound:
\begin{align*}
    \inf_{c > 0} \inf_{P \in \mathcal{P}_N} \|f_c - P\|_{L_\infty[a_c, b_c]} \ge C_3 N^{-\beta} \exp\left( - \frac{C_4 N}{\sqrt{T}} \right).
\end{align*}
Because the exponent $\sqrt{T}$ penalizes linear spacing, for practical long-context memory regimes where $N \ll \sqrt{T}$, the exponential factor is neutralized, rendering the observed scaling algebraic $\Omega(N^{-\beta})$.
\end{proof}

\begin{theorem}[Minimax Optimality of Geometric Spacing, Formal Version of Theorem~\ref{thm:optimality}]
\label{thm:optimality:formal}
There exists a configuration with geometrically spaced decay rates (i.e., uniformly spaced log-decay rates $p_k = \bar{G} (k-1) + p_1$) achieving the un-degraded optimal exponential rate:
\begin{align*}
    E_N(K) \le C_5 \exp\left( - \frac{\pi^2 N}{\log T + C_6} \right),
\end{align*}
where $C_5, C_6 > 0$ depend on $\beta$ but not on $N$. Furthermore, by Gonchar-Rakhmanov theory, a geometric progression $p_{k+1}^* - p_k^* \approx \text{const}$ is asymptotically necessary to attain this minimax limit.
\end{theorem}

\begin{proof}
The proof proceeds in three steps: (1) reduce exponential-sum approximation to rational approximation via the Laplace transform; (2) apply the Gonchar--Rakhmanov theory to establish exponential convergence with geometrically spaced nodes; (3) translate back to the exponential-sum setting.

\emph{Step 1: Laplace transform reduction.}
The power-law kernel admits the integral representation
\begin{align}
    K(t) = t^{-\beta} = \frac{1}{\Gamma(\beta)} \int_0^\infty s^{\beta - 1} e^{-st} \d s , \qquad t > 0.
    \label{eq:laplace_rep}
\end{align}
An $N$-term exponential sum $g(t) = \sum_{k=1}^N w_k e^{-p_k t}$ with $p_k > 0$ is the discrete analogue of this integral: it replaces the continuous measure $\frac{s^{\beta-1}}{\Gamma(\beta)}\d s$ by the atomic measure $\sum_k w_k \delta_{p_k}$.  Approximating $K$ on $[1,T]$ by $g$ is therefore equivalent to choosing $N$ nodes $\{p_k\}$ and weights $\{w_k\}$ such that the discrete quadrature approximates the Laplace integral uniformly for $t \in [1,T]$.

\emph{Step 2: Reduction to rational approximation.}
Setting $z = e^{-t}$, the interval $t \in [1, T]$ maps to $z \in [e^{-T}, e^{-1}]$.  In the $z$-domain, the target becomes $K(z) = (-\log z)^{-\beta}$ and the exponential sum becomes $g(z) = \sum_{k=1}^N w_k z^{p_k}$.  Alternatively, via the substitution $\lambda = e^{-1/p_k}$, the approximant takes the form of a generalized rational function.  The key connection is that the best exponential-sum approximation of $t^{-\beta}$ on $[1,T]$ is equivalent to the best type-$(N, 0)$ rational approximation of $s^{\beta - 1}$ on the spectral interval $[\Lambda_{\min}, \Lambda_{\max}]$ where $\Lambda_{\min} = 1/T$ and $\Lambda_{\max} = 1$, up to a linear change of variables~\citep[Section~3]{bm05}.

\emph{Step 3: Applying Gonchar--Rakhmanov theory.}
By the theorem of Gonchar and Rakhmanov~\citep{gr89}, the minimax error for best rational approximation of order $N$ to a function with algebraic branch-point singularities on a real interval $[a,b]$ with $0 < a < b$ satisfies
\begin{align}
    E_N^{\mathrm{rat}} \le C \exp\left( -\frac{\pi^2 N}{\log(b/a) + O(1)} \right),
    \label{eq:gonchar_rate}
\end{align}
where the constant in the exponent is determined by the logarithmic capacity of the condenser $(\{0\}, [a,b])$ in the complex plane.  Crucially, the optimal poles (Zolotarev nodes) are asymptotically equidistributed with respect to the logarithmic (harmonic) measure on $[a,b]$, which on the positive real axis corresponds to uniform spacing in $\log p$.  In exponential-sum language, this means
\begin{align*}
    \log p_{k+1}^* - \log p_k^* \to \text{const}, \qquad \text{as } N \to \infty,
\end{align*}
i.e., the optimal decay rates are geometrically spaced.

Beylkin and Monz\'{o}n~\citep{bm05} give explicit constructions of exponential sums with geometrically spaced exponents achieving this rate.  Applying~\eqref{eq:gonchar_rate} with $a = \Lambda_{\min} = 1/T$ and $b = \Lambda_{\max} = 1$, we obtain $\log(b/a) = \log T$, yielding
\begin{align*}
    E_N(K) \le C_5 \exp\left( -\frac{\pi^2 N}{\log T + C_6} \right),
\end{align*}
where $C_5, C_6 > 0$ depend on $\beta$ but not on $N$.
\end{proof}

\section{MQAR Experiment Details}
\label{app:mqar}

We adopt the Zoology framework~\citep{aet+24} and follow the MQAR setup of Dao \& Gu~\citep{dg24} (Appendix~D.1).  Each sequence writes $K$ key--value pairs (vocabulary size $V{=}8{,}192$), pads to length $T$, then queries all $K$ keys; loss is computed only on value predictions.  

\paragraph{Training.}
All models use 2 layers, RMSNorm, no MLP, no positional encoding, and are trained in BF16 with AdamW (weight decay 0.1, gradient clip 1.0, linear LR decay, batch size $2^{18}$ tokens).  Training uses a four-stage curriculum at $T_{\mathrm{train}}{=}512$: $K \in \{16, 32, 64, 128\}$ with $2^{18}$ examples per stage (8 epochs total).  Learning rates are swept per architecture family (3 values each; see released configs).

\paragraph{State-size equalization.}
\label{app:mqar_state}
To ensure fair comparison, all architectures share the same head count $h$ at each $d_{\mathrm{model}}$.  For Mamba-2, $d_{\mathrm{state}} = d/(2h)$ so that state size matches the $d^2/h$ of the other architectures.  We evaluate three configurations: $(d{=}512, h{=}4)$ giving 64K state, $(d{=}512, h{=}8)$ giving 32K state, and $(d{=}256, h{=}4)$ giving 16K state.

\paragraph{Evaluation.}
All tests use $K = T/4$ pairs.  The $T{=}512$ condition is in-distribution; $T \in \{1024, 2048, 4096\}$ are out-of-distribution extrapolation tests ($2{\times}$--$8{\times}$ training length).  Each condition uses 3,000 examples.  We select the checkpoint maximizing the sum of accuracies across all four test lengths.

\paragraph{Results.}
\label{app:mqar_results}
Figure~\ref{fig:mqar_faceted} visualizes the per-length accuracy data reported in Table~\ref{tab:mqar}.

\begin{figure}[h]
    \centering
    \includegraphics[width=\linewidth]{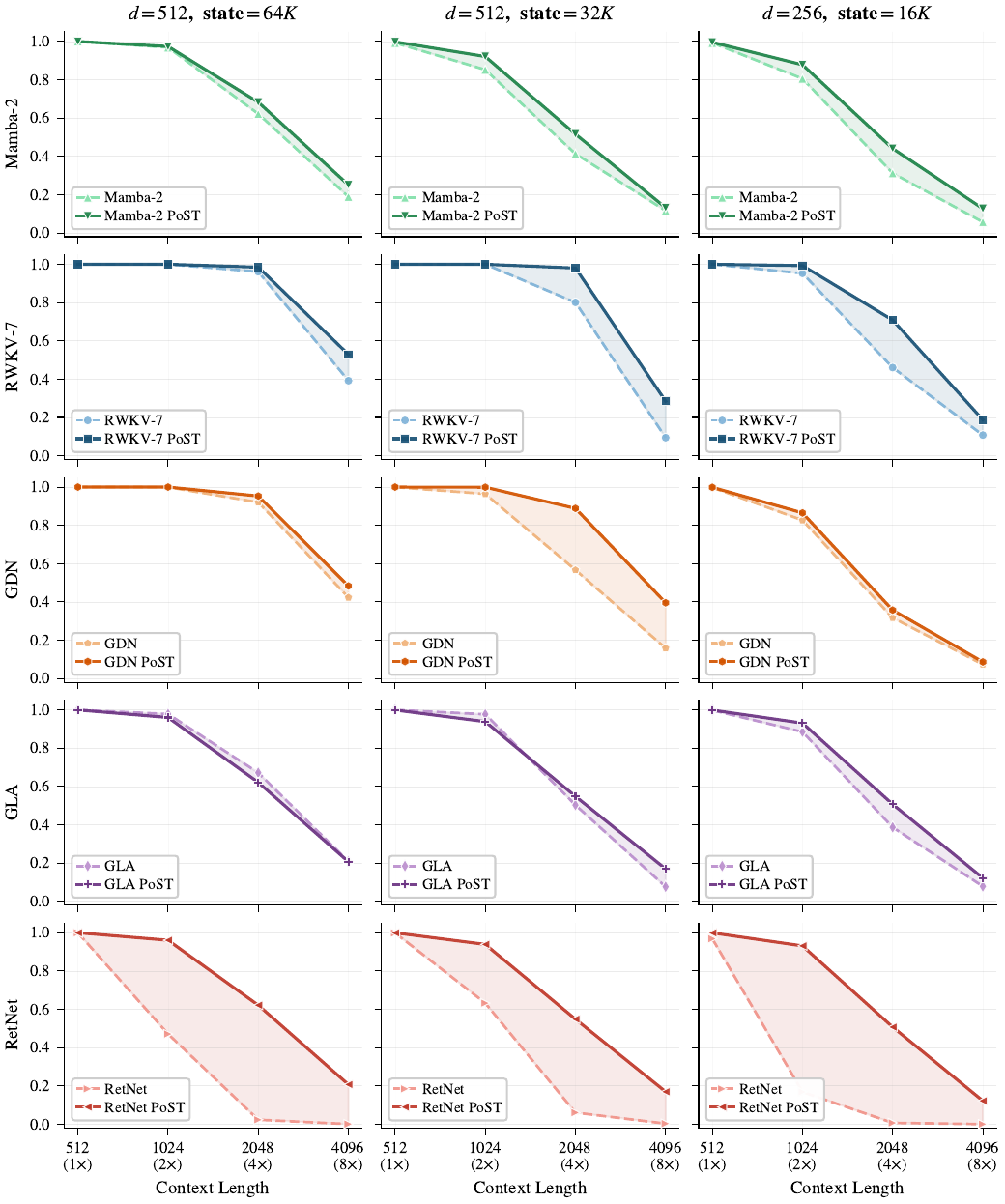}
    \caption{\textbf{MQAR extrapolation accuracy} across equalized per-layer state sizes $\in \{64K, 32K, 16K\}$ (from $d \in \{512, 256\}$ with varying head count).  Each point shows accuracy at $K = T/4$ key--value pairs for context length $T$.  All models trained at $T = 512$; longer lengths are out-of-distribution.  Solid lines: PoST variants; dashed: baselines.}
    \label{fig:mqar_faceted}
\end{figure}

\section{LM Evaluation Details}
\label{app:lm_eval}

This appendix provides the full experimental specification for the zero-shot language model evaluations reported in Section~\ref{sec:lm_pretrain}.

\subsection{Evaluation Framework}

We use the Language Model Evaluation Harness~\citep{gta+24} (version 0.4.x) to evaluate pretrained base models in a zero-shot setting.
Each task is cast as a log-likelihood ranking problem: the model scores candidate completions and selects the one with the highest probability under the language model.
No in-context learning examples (few-shot) or instruction tuning are used.

\subsection{Model Architecture and Training}
\label{app:lm_eval_models}

Within each model pair at a given scale, the models share an identical architecture and differ only in SSM/decay parameterization.  Table~\ref{tab:lm_mamba2_config} summarizes the Mamba-2 architecture used in the LM evaluation experiments.

\begin{table}[h]
\centering
\caption{Mamba-2 model configurations for LM evaluation.}
\label{tab:lm_mamba2_config}
\small
\begin{tabular}{@{}lll@{}}
\toprule
\textbf{Parameter} & \textbf{180M} & \textbf{440M} \\
\midrule
$d_{\mathrm{model}}$ & 768 & 1,024 \\
$d_{\mathrm{inner}}$ ($= \mathrm{expand} \times d_{\mathrm{model}}$) & 1,536 & 2,048 \\
Number of layers $L$ & 24 & 48 \\
$d_{\mathrm{state}}$ & 128 & 128 \\
Head dimension & 64 & 64 \\
Number of heads $h$ & 24 & 32 \\
Convolution width $d_{\mathrm{conv}}$ & 4 & 4 \\
Expand factor & 2 & 2 \\
Chunk size (SSD) & 256 & 256 \\
Vocabulary size & 128,256 & 128,256 \\
Tied embeddings & Yes & Yes \\
\bottomrule
\end{tabular}
\end{table}

\begin{table}[h]
\centering
\caption{Training configuration.}
\label{tab:lm_train_config}
\small
\begin{tabular}{@{}ll@{}}
\toprule
\textbf{Parameter} & \textbf{Value} \\
\midrule
Training data & FineWeb-Edu~\citep{lbds24} \\
Training context length $T_{\mathrm{train}}$ & 2,048 \\
Tokenizer & Llama-3.1~\citep{dma+24} \\
Hardware & $8\times$ H200-SXM \\
Optimizer & AdamW~\citep{lh19} \\
Warmup & 1\% of total steps \\
Precision & BF16 mixed precision \\
Gradient clipping & $\|\nabla\|_2 \le 1.0$ \\
\midrule
\multicolumn{2}{@{}l}{\emph{Scale-dependent}} \\
Training tokens (180M / 440M) & $4$B / $9$B \\
Learning rate (180M / 440M) & $6 \times 10^{-4}$ / $3 \times 10^{-4}$ (cosine, min lr $= 10^{-5}$) \\
\midrule
\multicolumn{2}{@{}l}{\emph{Mamba-2 specific}} \\
$\beta_1, \beta_2$ & $0.9, 0.95$ \\
$\varepsilon$ & $10^{-8}$ \\
Weight decay & $0.1$ \\
\bottomrule
\end{tabular}
\end{table}

\noindent\textbf{Note on RWKV-7 optimizer settings.} Following the official RWKV-7 training recipe, the RWKV-7 models use $\beta_2 = 0.99$ and $\varepsilon = 10^{-18}$ instead of the Mamba-2 values above.  All other optimizer and scheduler settings are shared.

Table~\ref{tab:lm_mamba2_init} shows the initialization comparison for the Mamba-2 model pair.

\begin{table}[h]
\centering
\caption{Mamba-2 initialization comparison.}
\label{tab:lm_mamba2_init}
\small
\begin{tabular}{@{}lll@{}}
\toprule
& \textbf{Mamba-2 (Baseline)} & \textbf{Mamba-2 PoST} \\
\midrule
A initialization & S4D-Real: $\lambda_k = -(k{+}1)$ & Geometric (Def.~\ref{def:post_map}) \\
Timescale range at $T_{\mathrm{train}}$ & uncontrolled & $[1, T_{\mathrm{train}}]$ (dynamic: $[1, t]$ at position $t$) \\
$\Delta_t$ initialization & Random (default) & Fixed: $0.05$ \\
Position adaptive & No & Yes \\
\bottomrule
\end{tabular}
\end{table}

\begin{table}[h]
\centering
\caption{RWKV-7 initialization comparison.}
\label{tab:lm_rwkv7_init}
\small
\begin{tabular}{@{}lll@{}}
\toprule
& \textbf{RWKV-7 (Baseline)} & \textbf{RWKV-7 PoST} \\
\midrule
Decay bias init & Power-law + zigzag (official) & Geometric (Def.~\ref{def:post_map}) \\
$w_{0,k}$ range & $[-5.5, 5.5]$ (zigzag) & Eq.~\ref{eq:rwkv_init} (increasing) \\
Timescale range at $T_{\mathrm{train}}$ & uncontrolled (layer-dep.) & $[1, T_{\mathrm{train}}]$ (dynamic: $[1, t]$ at position $t$) \\
Taper exponents & --- & Cor.~\ref{cor:sigmoid_taper}; $\alpha_0=1$ (slow), $\alpha_{C-1}=0$ (fast) \\
Position adaptive & No & Yes \\
\bottomrule
\end{tabular}
\end{table}

\noindent\textbf{Note on RWKV-7 PoST.}
PoST replaces the official power-law initialization with Spectral Reparameterization in logit space, subtracts $\alpha_k \log t$ inside the logit (Eq.~\ref{eq:rwkv_post_additive}), and retains the zigzag LoRA bias for intra-head variation.
Full implementation details are available in the open-source code.

\section{PoST-RetNet / GLA Pseudocode}
\label{app:arch_pseudocode}

This appendix provides the forward-pass pseudocode for PoST-RetNet, complementing the Mamba-2 (Section~\ref{subsec:mamba_post}) and RWKV-7 (Section~\ref{subsec:rwkv_post}) instantiations in the main body.

RetNet~\citep{syc+23} uses a fixed per-head scalar decay $\gamma_h \in (0, 1)$, typically initialized as $\gamma_h = 1 - 2^{-(5 + h \cdot 3/(H-1))}$.  Because GLA shares the same per-head scalar decay structure, applying PoST to GLA yields an identical reparameterization; accordingly, PoST-RetNet and PoST-GLA reduce to the same model and are reported together in our experiments (Table~\ref{tab:mqar}).

\begin{algorithm}[!ht]
\caption{PoST-RetNet / GLA: Retention Forward Pass}
\label{alg:retnet_post}
\begin{algorithmic}[1]
\Require Input $x \in \R^{B \times T \times D}$, learnable parameters $\theta_\gamma \in \R$, $\delta_\gamma \in \R^{H-1}$, RetNet projection weights, position offset $t_0 \ge 0$.
\Ensure Output $y \in \R^{B \times T \times D}$
\Statex
\State {\color{blue} /* PoST map for retention decay (replaces hand-crafted $\gamma_h$) */}
\State $g_j \gets \softplus(\delta_{\gamma,j})$ for $j = 1, \ldots, H-1$ \Comment{Definition~\ref{def:post_map}}
\State $p_h \gets \theta_\gamma + \sum_{j=1}^{h-1} g_j$ for $h = 1, \ldots, H$ \Comment{Ordered log-decay rates}
\State $\gamma_h \gets \exp(-\exp(p_h))$ for $h = 1, \ldots, H$ \Comment{$\gamma_h \in (0, 1)$, geometrically spaced}
\Statex
\State {\color{blue} /* Standard RetNet projections (unchanged) */}
\State $Q, K, V \gets W_Q x, W_K x, W_V x$ \Comment{Multi-head projections}
\Statex
\State {\color{blue} /* Position-adaptive scaling via effective decay */}
\State $\bar{G} \gets (p_H - p_1) / (H - 1)$ \Comment{Mean spectral gap}
\State $\alpha_h \gets \clamp\bigl(\tfrac{H-h}{H-1} + \tfrac{(p_h - p_1) - (h-1)\bar{G}}{\log T_{\mathrm{train}}}, 0, 1\bigr)$ for $h = 1, \ldots, H$ \Comment{Proposition~\ref{thm:taper_general}}
\State $\mathbf{t} \gets [t_0 + 1, t_0 + 2, \ldots, t_0 + T]$
\State $\gamma_{h,l} \gets \gamma_h^{\mathbf{t}_l^{-\alpha_h} }$ for $l \in [T], h \in [H]$ \Comment{$= \exp(-\exp(p_h) / \mathbf{t}_l^{\alpha_h})$}
\Statex
\State {\color{blue} /* Retention recurrence */}
\State $S_0 \gets \mathbf{0}$
\For{$t = 1, \ldots, T$}
    \State $S_t^{(h)} \gets \gamma_{h,t} \cdot S_{t-1}^{(h)} + K_t^{(h)\top} V_t^{(h)}$ for $h \in [H]$
    \State $y_t^{(h)} \gets Q_t^{(h)} \cdot S_t^{(h)}$ for $h \in [H]$
\EndFor
\State \Return $y$
\end{algorithmic}
\end{algorithm}

\paragraph{Remark.}
Standard RetNet uses constant $\gamma_h$ across all positions.  The PoST modification makes the effective $\gamma$ position-dependent (via the position-adaptive decay $\gamma_{h,l}$ in Algorithm~\ref{alg:retnet_post}) while preserving the chunk-parallel retention computation: within each chunk, $\gamma_{h,l}$ varies smoothly and the retention matrix remains lower-triangular with known structure.

\ifdefined\isarxiv
\else
\newpage
\section*{NeurIPS Paper Checklist}

The checklist is designed to encourage best practices for responsible machine learning research, addressing issues of reproducibility, transparency, research ethics, and societal impact. Do not remove the checklist: {\bf The papers not including the checklist will be desk rejected.} The checklist should follow the references and follow the (optional) supplemental material.  The checklist does NOT count towards the page
limit. 

Please read the checklist guidelines carefully for information on how to answer these questions. For each question in the checklist:
\begin{itemize}
    \item You should answer \answerYes{}, \answerNo{}, or \answerNA{}.
    \item \answerNA{} means either that the question is Not Applicable for that particular paper or the relevant information is Not Available.
    \item Please provide a short (1--2 sentence) justification right after your answer (even for \answerNA). 
\end{itemize}

{\bf The checklist answers are an integral part of your paper submission.} They are visible to the reviewers, area chairs, senior area chairs, and ethics reviewers. You will also be asked to include it (after eventual revisions) with the final version of your paper, and its final version will be published with the paper.

The reviewers of your paper will be asked to use the checklist as one of the factors in their evaluation. While \answerYes{} is generally preferable to \answerNo{}, it is perfectly acceptable to answer \answerNo{} provided a proper justification is given (e.g., error bars are not reported because it would be too computationally expensive'' or ``we were unable to find the license for the dataset we used''). In general, answering \answerNo{} or \answerNA{} is not grounds for rejection. While the questions are phrased in a binary way, we acknowledge that the true answer is often more nuanced, so please just use your best judgment and write a justification to elaborate. All supporting evidence can appear either in the main paper or the supplemental material, provided in appendix. If you answer \answerYes{} to a question, in the justification please point to the section(s) where related material for the question can be found.

IMPORTANT, please:
\begin{itemize}
    \item {\bf Delete this instruction block, but keep the section heading ``NeurIPS Paper Checklist"},
    \item  {\bf Keep the checklist subsection headings, questions/answers and guidelines below.}
    \item {\bf Do not modify the questions and only use the provided macros for your answers}.
\end{itemize}


\begin{enumerate}

\item {\bf Claims}
    \item[] Question: Do the main claims made in the abstract and introduction accurately reflect the paper's contributions and scope?
    \item[] Answer: \answerTODO{} 
    \item[] Justification: \justificationTODO{}
    \item[] Guidelines:
    \begin{itemize}
        \item The answer \answerNA{} means that the abstract and introduction do not include the claims made in the paper.
        \item The abstract and/or introduction should clearly state the claims made, including the contributions made in the paper and important assumptions and limitations. A \answerNo{} or \answerNA{} answer to this question will not be perceived well by the reviewers. 
        \item The claims made should match theoretical and experimental results, and reflect how much the results can be expected to generalize to other settings. 
        \item It is fine to include aspirational goals as motivation as long as it is clear that these goals are not attained by the paper. 
    \end{itemize}

\item {\bf Limitations}
    \item[] Question: Does the paper discuss the limitations of the work performed by the authors?
    \item[] Answer: \answerTODO{} 
    \item[] Justification: \justificationTODO{}
    \item[] Guidelines:
    \begin{itemize}
        \item The answer \answerNA{} means that the paper has no limitation while the answer \answerNo{} means that the paper has limitations, but those are not discussed in the paper. 
        \item The authors are encouraged to create a separate ``Limitations'' section in their paper.
        \item The paper should point out any strong assumptions and how robust the results are to violations of these assumptions (e.g., independence assumptions, noiseless settings, model well-specification, asymptotic approximations only holding locally). The authors should reflect on how these assumptions might be violated in practice and what the implications would be.
        \item The authors should reflect on the scope of the claims made, e.g., if the approach was only tested on a few datasets or with a few runs. In general, empirical results often depend on implicit assumptions, which should be articulated.
        \item The authors should reflect on the factors that influence the performance of the approach. For example, a facial recognition algorithm may perform poorly when image resolution is low or images are taken in low lighting. Or a speech-to-text system might not be used reliably to provide closed captions for online lectures because it fails to handle technical jargon.
        \item The authors should discuss the computational efficiency of the proposed algorithms and how they scale with dataset size.
        \item If applicable, the authors should discuss possible limitations of their approach to address problems of privacy and fairness.
        \item While the authors might fear that complete honesty about limitations might be used by reviewers as grounds for rejection, a worse outcome might be that reviewers discover limitations that aren't acknowledged in the paper. The authors should use their best judgment and recognize that individual actions in favor of transparency play an important role in developing norms that preserve the integrity of the community. Reviewers will be specifically instructed to not penalize honesty concerning limitations.
    \end{itemize}

\item {\bf Theory assumptions and proofs}
    \item[] Question: For each theoretical result, does the paper provide the full set of assumptions and a complete (and correct) proof?
    \item[] Answer: \answerTODO{} 
    \item[] Justification: \justificationTODO{}
    \item[] Guidelines:
    \begin{itemize}
        \item The answer \answerNA{} means that the paper does not include theoretical results. 
        \item All the theorems, formulas, and proofs in the paper should be numbered and cross-referenced.
        \item All assumptions should be clearly stated or referenced in the statement of any theorems.
        \item The proofs can either appear in the main paper or the supplemental material, but if they appear in the supplemental material, the authors are encouraged to provide a short proof sketch to provide intuition. 
        \item Inversely, any informal proof provided in the core of the paper should be complemented by formal proofs provided in appendix or supplemental material.
        \item Theorems and Lemmas that the proof relies upon should be properly referenced. 
    \end{itemize}

    \item {\bf Experimental result reproducibility}
    \item[] Question: Does the paper fully disclose all the information needed to reproduce the main experimental results of the paper to the extent that it affects the main claims and/or conclusions of the paper (regardless of whether the code and data are provided or not)?
    \item[] Answer: \answerTODO{} 
    \item[] Justification: \justificationTODO{}
    \item[] Guidelines:
    \begin{itemize}
        \item The answer \answerNA{} means that the paper does not include experiments.
        \item If the paper includes experiments, a \answerNo{} answer to this question will not be perceived well by the reviewers: Making the paper reproducible is important, regardless of whether the code and data are provided or not.
        \item If the contribution is a dataset and\slash or model, the authors should describe the steps taken to make their results reproducible or verifiable. 
        \item Depending on the contribution, reproducibility can be accomplished in various ways. For example, if the contribution is a novel architecture, describing the architecture fully might suffice, or if the contribution is a specific model and empirical evaluation, it may be necessary to either make it possible for others to replicate the model with the same dataset, or provide access to the model. In general. releasing code and data is often one good way to accomplish this, but reproducibility can also be provided via detailed instructions for how to replicate the results, access to a hosted model (e.g., in the case of a large language model), releasing of a model checkpoint, or other means that are appropriate to the research performed.
        \item While NeurIPS does not require releasing code, the conference does require all submissions to provide some reasonable avenue for reproducibility, which may depend on the nature of the contribution. For example
        \begin{enumerate}
            \item If the contribution is primarily a new algorithm, the paper should make it clear how to reproduce that algorithm.
            \item If the contribution is primarily a new model architecture, the paper should describe the architecture clearly and fully.
            \item If the contribution is a new model (e.g., a large language model), then there should either be a way to access this model for reproducing the results or a way to reproduce the model (e.g., with an open-source dataset or instructions for how to construct the dataset).
            \item We recognize that reproducibility may be tricky in some cases, in which case authors are welcome to describe the particular way they provide for reproducibility. In the case of closed-source models, it may be that access to the model is limited in some way (e.g., to registered users), but it should be possible for other researchers to have some path to reproducing or verifying the results.
        \end{enumerate}
    \end{itemize}

\item {\bf Open access to data and code}
    \item[] Question: Does the paper provide open access to the data and code, with sufficient instructions to faithfully reproduce the main experimental results, as described in supplemental material?
    \item[] Answer: \answerTODO{} 
    \item[] Justification: \justificationTODO{}
    \item[] Guidelines:
    \begin{itemize}
        \item The answer \answerNA{} means that paper does not include experiments requiring code.
        \item Please see the NeurIPS code and data submission guidelines (\url{https://neurips.cc/public/guides/CodeSubmissionPolicy}) for more details.
        \item While we encourage the release of code and data, we understand that this might not be possible, so \answerNo{} is an acceptable answer. Papers cannot be rejected simply for not including code, unless this is central to the contribution (e.g., for a new open-source benchmark).
        \item The instructions should contain the exact command and environment needed to run to reproduce the results. See the NeurIPS code and data submission guidelines (\url{https://neurips.cc/public/guides/CodeSubmissionPolicy}) for more details.
        \item The authors should provide instructions on data access and preparation, including how to access the raw data, preprocessed data, intermediate data, and generated data, etc.
        \item The authors should provide scripts to reproduce all experimental results for the new proposed method and baselines. If only a subset of experiments are reproducible, they should state which ones are omitted from the script and why.
        \item At submission time, to preserve anonymity, the authors should release anonymized versions (if applicable).
        \item Providing as much information as possible in supplemental material (appended to the paper) is recommended, but including URLs to data and code is permitted.
    \end{itemize}

\item {\bf Experimental setting/details}
    \item[] Question: Does the paper specify all the training and test details (e.g., data splits, hyperparameters, how they were chosen, type of optimizer) necessary to understand the results?
    \item[] Answer: \answerTODO{} 
    \item[] Justification: \justificationTODO{}
    \item[] Guidelines:
    \begin{itemize}
        \item The answer \answerNA{} means that the paper does not include experiments.
        \item The experimental setting should be presented in the core of the paper to a level of detail that is necessary to appreciate the results and make sense of them.
        \item The full details can be provided either with the code, in appendix, or as supplemental material.
    \end{itemize}

\item {\bf Experiment statistical significance}
    \item[] Question: Does the paper report error bars suitably and correctly defined or other appropriate information about the statistical significance of the experiments?
    \item[] Answer: \answerTODO{} 
    \item[] Justification: \justificationTODO{}
    \item[] Guidelines:
    \begin{itemize}
        \item The answer \answerNA{} means that the paper does not include experiments.
        \item The authors should answer \answerYes{} if the results are accompanied by error bars, confidence intervals, or statistical significance tests, at least for the experiments that support the main claims of the paper.
        \item The factors of variability that the error bars are capturing should be clearly stated (for example, train/test split, initialization, random drawing of some parameter, or overall run with given experimental conditions).
        \item The method for calculating the error bars should be explained (closed form formula, call to a library function, bootstrap, etc.)
        \item The assumptions made should be given (e.g., Normally distributed errors).
        \item It should be clear whether the error bar is the standard deviation or the standard error of the mean.
        \item It is OK to report 1-sigma error bars, but one should state it. The authors should preferably report a 2-sigma error bar than state that they have a 96\% CI, if the hypothesis of Normality of errors is not verified.
        \item For asymmetric distributions, the authors should be careful not to show in tables or figures symmetric error bars that would yield results that are out of range (e.g., negative error rates).
        \item If error bars are reported in tables or plots, the authors should explain in the text how they were calculated and reference the corresponding figures or tables in the text.
    \end{itemize}

\item {\bf Experiments compute resources}
    \item[] Question: For each experiment, does the paper provide sufficient information on the computer resources (type of compute workers, memory, time of execution) needed to reproduce the experiments?
    \item[] Answer: \answerTODO{} 
    \item[] Justification: \justificationTODO{}
    \item[] Guidelines:
    \begin{itemize}
        \item The answer \answerNA{} means that the paper does not include experiments.
        \item The paper should indicate the type of compute workers CPU or GPU, internal cluster, or cloud provider, including relevant memory and storage.
        \item The paper should provide the amount of compute required for each of the individual experimental runs as well as estimate the total compute. 
        \item The paper should disclose whether the full research project required more compute than the experiments reported in the paper (e.g., preliminary or failed experiments that didn't make it into the paper). 
    \end{itemize}
    
\item {\bf Code of ethics}
    \item[] Question: Does the research conducted in the paper conform, in every respect, with the NeurIPS Code of Ethics \url{https://neurips.cc/public/EthicsGuidelines}?
    \item[] Answer: \answerTODO{} 
    \item[] Justification: \justificationTODO{}
    \item[] Guidelines:
    \begin{itemize}
        \item The answer \answerNA{} means that the authors have not reviewed the NeurIPS Code of Ethics.
        \item If the authors answer \answerNo, they should explain the special circumstances that require a deviation from the Code of Ethics.
        \item The authors should make sure to preserve anonymity (e.g., if there is a special consideration due to laws or regulations in their jurisdiction).
    \end{itemize}

\item {\bf Broader impacts}
    \item[] Question: Does the paper discuss both potential positive societal impacts and negative societal impacts of the work performed?
    \item[] Answer: \answerTODO{} 
    \item[] Justification: \justificationTODO{}
    \item[] Guidelines:
    \begin{itemize}
        \item The answer \answerNA{} means that there is no societal impact of the work performed.
        \item If the authors answer \answerNA{} or \answerNo, they should explain why their work has no societal impact or why the paper does not address societal impact.
        \item Examples of negative societal impacts include potential malicious or unintended uses (e.g., disinformation, generating fake profiles, surveillance), fairness considerations (e.g., deployment of technologies that could make decisions that unfairly impact specific groups), privacy considerations, and security considerations.
        \item The conference expects that many papers will be foundational research and not tied to particular applications, let alone deployments. However, if there is a direct path to any negative applications, the authors should point it out. For example, it is legitimate to point out that an improvement in the quality of generative models could be used to generate Deepfakes for disinformation. On the other hand, it is not needed to point out that a generic algorithm for optimizing neural networks could enable people to train models that generate Deepfakes faster.
        \item The authors should consider possible harms that could arise when the technology is being used as intended and functioning correctly, harms that could arise when the technology is being used as intended but gives incorrect results, and harms following from (intentional or unintentional) misuse of the technology.
        \item If there are negative societal impacts, the authors could also discuss possible mitigation strategies (e.g., gated release of models, providing defenses in addition to attacks, mechanisms for monitoring misuse, mechanisms to monitor how a system learns from feedback over time, improving the efficiency and accessibility of ML).
    \end{itemize}
    
\item {\bf Safeguards}
    \item[] Question: Does the paper describe safeguards that have been put in place for responsible release of data or models that have a high risk for misuse (e.g., pre-trained language models, image generators, or scraped datasets)?
    \item[] Answer: \answerTODO{} 
    \item[] Justification: \justificationTODO{}
    \item[] Guidelines:
    \begin{itemize}
        \item The answer \answerNA{} means that the paper poses no such risks.
        \item Released models that have a high risk for misuse or dual-use should be released with necessary safeguards to allow for controlled use of the model, for example by requiring that users adhere to usage guidelines or restrictions to access the model or implementing safety filters. 
        \item Datasets that have been scraped from the Internet could pose safety risks. The authors should describe how they avoided releasing unsafe images.
        \item We recognize that providing effective safeguards is challenging, and many papers do not require this, but we encourage authors to take this into account and make a best faith effort.
    \end{itemize}

\item {\bf Licenses for existing assets}
    \item[] Question: Are the creators or original owners of assets (e.g., code, data, models), used in the paper, properly credited and are the license and terms of use explicitly mentioned and properly respected?
    \item[] Answer: \answerTODO{} 
    \item[] Justification: \justificationTODO{}
    \item[] Guidelines:
    \begin{itemize}
        \item The answer \answerNA{} means that the paper does not use existing assets.
        \item The authors should cite the original paper that produced the code package or dataset.
        \item The authors should state which version of the asset is used and, if possible, include a URL.
        \item The name of the license (e.g., CC-BY 4.0) should be included for each asset.
        \item For scraped data from a particular source (e.g., website), the copyright and terms of service of that source should be provided.
        \item If assets are released, the license, copyright information, and terms of use in the package should be provided. For popular datasets, \url{paperswithcode.com/datasets} has curated licenses for some datasets. Their licensing guide can help determine the license of a dataset.
        \item For existing datasets that are re-packaged, both the original license and the license of the derived asset (if it has changed) should be provided.
        \item If this information is not available online, the authors are encouraged to reach out to the asset's creators.
    \end{itemize}

\item {\bf New assets}
    \item[] Question: Are new assets introduced in the paper well documented and is the documentation provided alongside the assets?
    \item[] Answer: \answerTODO{} 
    \item[] Justification: \justificationTODO{}
    \item[] Guidelines:
    \begin{itemize}
        \item The answer \answerNA{} means that the paper does not release new assets.
        \item Researchers should communicate the details of the dataset\slash code\slash model as part of their submissions via structured templates. This includes details about training, license, limitations, etc. 
        \item The paper should discuss whether and how consent was obtained from people whose asset is used.
        \item At submission time, remember to anonymize your assets (if applicable). You can either create an anonymized URL or include an anonymized zip file.
    \end{itemize}

\item {\bf Crowdsourcing and research with human subjects}
    \item[] Question: For crowdsourcing experiments and research with human subjects, does the paper include the full text of instructions given to participants and screenshots, if applicable, as well as details about compensation (if any)? 
    \item[] Answer: \answerTODO{} 
    \item[] Justification: \justificationTODO{}
    \item[] Guidelines:
    \begin{itemize}
        \item The answer \answerNA{} means that the paper does not involve crowdsourcing nor research with human subjects.
        \item Including this information in the supplemental material is fine, but if the main contribution of the paper involves human subjects, then as much detail as possible should be included in the main paper. 
        \item According to the NeurIPS Code of Ethics, workers involved in data collection, curation, or other labor should be paid at least the minimum wage in the country of the data collector. 
    \end{itemize}

\item {\bf Institutional review board (IRB) approvals or equivalent for research with human subjects}
    \item[] Question: Does the paper describe potential risks incurred by study participants, whether such risks were disclosed to the subjects, and whether Institutional Review Board (IRB) approvals (or an equivalent approval/review based on the requirements of your country or institution) were obtained?
    \item[] Answer: \answerTODO{} 
    \item[] Justification: \justificationTODO{}
    \item[] Guidelines:
    \begin{itemize}
        \item The answer \answerNA{} means that the paper does not involve crowdsourcing nor research with human subjects.
        \item Depending on the country in which research is conducted, IRB approval (or equivalent) may be required for any human subjects research. If you obtained IRB approval, you should clearly state this in the paper. 
        \item We recognize that the procedures for this may vary significantly between institutions and locations, and we expect authors to adhere to the NeurIPS Code of Ethics and the guidelines for their institution. 
        \item For initial submissions, do not include any information that would break anonymity (if applicable), such as the institution conducting the review.
    \end{itemize}

\item {\bf Declaration of LLM usage}
    \item[] Question: Does the paper describe the usage of LLMs if it is an important, original, or non-standard component of the core methods in this research? Note that if the LLM is used only for writing, editing, or formatting purposes and does \emph{not} impact the core methodology, scientific rigor, or originality of the research, declaration is not required.
    \item[] Answer: \answerTODO{} 
    \item[] Justification: \justificationTODO{}
    \item[] Guidelines:
    \begin{itemize}
        \item The answer \answerNA{} means that the core method development in this research does not involve LLMs as any important, original, or non-standard components.
        \item Please refer to our LLM policy in the NeurIPS handbook for what should or should not be described.
    \end{itemize}

\end{enumerate}
\fi

\end{document}